\documentclass[twoside]{article}

\usepackage[accepted]{aistats2020}
\usepackage[T1]{fontenc}

\usepackage[round]{natbib}

\usepackage{amsmath,bm,bbm,mathtools}
\usepackage{amsthm}
\usepackage[autostyle]{csquotes}
\usepackage{subcaption,multirow}
\usepackage{amssymb,pifont}
\usepackage{hyperref}       
\usepackage{url}            
\usepackage[scr=boondox]{mathalfa}
\usepackage{booktabs}       
\usepackage{amsfonts}       
\usepackage{nicefrac}       
\usepackage{microtype}      
\usepackage[colorinlistoftodos]{todonotes}
\usepackage{soul}
\usepackage{xcolor,colortbl}
\usepackage{booktabs}
\usepackage{algorithm}
\usepackage{algorithmic}

\definecolor{anti-flashwhite}{rgb}{0.95, 0.95, 0.96}
\definecolor{antiquewhite}{rgb}{0.98, 0.92, 0.84}

\newcolumntype{g}{>{\columncolor{gray}}c}
\newcolumntype{w}{>{\columncolor{white}}c}
\newcolumntype{d}{>{\columncolor{anti-flashwhite}}c}
\newcolumntype{a}{>{\columncolor{antiquewhite}}c}

\def\eqref#1{equation~\ref{#1}}

\def\1{\bm{1}}

\DeclareMathAlphabet{\mathsfit}{\encodingdefault}{\sfdefault}{m}{sl}
\SetMathAlphabet{\mathsfit}{bold}{\encodingdefault}{\sfdefault}{bx}{n}

\newcommand{\ind}{\mathbbm{1}}

\DeclareMathOperator*{\argmin}{arg\,min}

\newtheorem{theorem}{Theorem}

\newtheorem{condition}{Condition}

\begin{document}

%

%

\twocolumn[

\aistatstitle{Censored Quantile Regression Forest}

\aistatsauthor{Alexander Hanbo Li \footnotemark[1] \And Jelena Bradic}
\aistatsaddress{Alexa AI, Seattle \And  University of California San Diego}

]

\footnotetext[1]{Work mostly done at University of California San Diego}

\begin{abstract}
  Random forests are powerful non-parametric regression method but are severely limited in their usage in the presence of randomly censored observations, and naively applied can exhibit poor predictive performance due to the incurred biases. Based on a local adaptive representation of random forests, we develop its regression adjustment for randomly censored regression quantile models. Regression adjustment is based on a new estimating equation that adapts to censoring and leads to quantile score whenever the data do not exhibit censoring. The proposed procedure named {\it censored quantile regression forest}, allows us to estimate quantiles of time-to-event without any parametric modeling assumption. We establish its consistency under mild model specifications. Numerical studies showcase a clear advantage of the proposed procedure.
\end{abstract}

\section{Introduction}
\label{sec:intro}

Censored data exists in many different areas. In economics, policies such as minimum wage and minimum transaction fee result in left-censored data. In biomedical study, researchers cannot always observe the time until the occurrence of an event of interest because the time span of the study is limited or the patient withdraws from the experiment, resulting in right-censored data.

Classical statistical approaches always assume an underlying model like accelerated failure time model \citep{koul1981regression,robins1992semiparametric,robins1992estimation,wei1992accelerated,zeng2007efficient,huang2007least}. These methods perform well when the model is correctly specified, but quickly break down when the model assumption is wrong or the error distribution is heteroscedastic. Other non-parametric methods \citep{doi:10.1093/biomet/68.2.381,doi:10.1093/biomet/85.4.809} and rank-based methods \citep{jin2003rank} strive to achieve assumption-lean modeling of the mean. Forest algorithms \citep{breiman2001random,geurts2006extremely,meinshausen2006quantile,athey2019generalized} are non-parametric and allow for flexible modeling of covariate interactions. However, it is non-trivial to adapt forest algorithms to censored data. Random survival forests \citep{ishwaran2008random} or bagging survival trees \citep{hothorn2004bagging,hothorn2005survival} rely on building survival trees using survival function as the splitting criterion, and are only applicable to right-censored data. Moreover, any technique developed for uncensored data cannot be easily extended to censoring scenario. For example, generalized random forest can effectively deal with heteroscedastic data, but applying the same technique to heteroscedastic and censored data is non-trivial. 

In this paper, we propose a novel method that connects the quantile forest algorithms to the censored data problem. This is done by a carefully designed estimating equation. In this way, any technique developed for quantile forests on uncensored data can be seamlessly applied to censored data. One of the promising applications of the introduced method is in the estimation of heterogeneous treatment effects when the response variable is censored. We will show in the experiments that the introduced methods do achieve the best performance on both simulated and real datasets. Especially on one heteroscedastic data, the proposed method is the only working solution among many other forest algorithms, including random survival forest. The proposed method, \textit{censored quantile regression forest}, is motivated by the observation that random forests actually define a local similarity metric \citep{lin2006random,li2017forest,athey2019generalized} which is essentially a data-driven kernel. Using this kernel, random forests can be rephrased as locally weighted regressions. We will review the regression adjustments for forests in Section \ref{sec:adjustments}.

\subsection{Related Work}

In the case of right censoring, most non-parametric recursive partitioning algorithms rely on survival tree or its ensembles. \cite{ishwaran2008random} proposed random survival forest (RSF) algorithm in which each tree is built by maximizing the between-node log-rank statistic. However, it is not directly estimating the conditional quantiles but instead estimating the cumulative hazard. \cite{zhu2012recursively} proposed the recursively imputed survival trees (RIST) algorithm with the same splitting criterion for each individual tree but different ensemble scheme. Other similar methods relying on different kinds of survival trees were proposed in \cite{gordon1985tree}, \cite{segal1988regression}, \cite{davis1989exponential}, \cite{leblanc1992relative} and \cite{leblanc1993survival}. All these methods as mentioned above use splitting rules specifically designed for the right censored data, and they all rely on the proportional hazard assumption and cannot reduce to a loss-based method that might ordinarily be used in the situation with no censoring. \cite{molinaro2004tree} proposed a tree method based on the inverse probability censoring \citep{robins1994estimation} weighted (IPCW) loss function which reduces to the full data loss function used by CART in the absence of censoring. \cite{hothorn2005survival} then extended the IPCW idea and proposed a forest-type method in which each tree is trained on resampled observations according to inverse probability censoring weights. However, the censored data always get weights zero and hence only uncensored observations will be resampled. As pointed out by \cite{robins1994estimation}, the inverse probability weighted estimators are inefficient because of their failure to utilize all the information available on observations with missing or partially missing data. 
\section{Regression Adjustments for Forests}
\label{sec:adjustments}
We will briefly review random forest and generalized forest in this section and show that they can be written as weighted regression problems. We also introduce \textit{``forest weights"} that is an essential concept in this paper.

\paragraph{Random Forest}
\label{ssec:rf_weights}

Let $\theta$ denote the random parameter determining how a tree is grown, and $\{(X_i, Y_i): i=1,\ldots,n\} \in \mathcal{X} \times \mathcal{Y} \subset \mathbb{R}^p \times \mathbb{R}$ denote the training data. For each tree $T(\theta)$, let $R_l$ denotes its $l$-th terminal leaf. We let the index of the leaf that contains $x$ to be $l(x;\theta)$.

As shown in \citet{meinshausen2006quantile}, for any single tree $T(\theta)$, the prediction on $x$ can be written as $\sum_{i=1}^n w(X_i, x; \theta) Y_i$ where $w(X_i, x; \theta) = \ind_{\{ X_i \in R_{l(x;\theta)} \}} / \# \{ j: X_j \in R_{l(x;\theta)} \}$. Then  a random forest containing $m$ trees formulates a  prediction of $\mathbb{E}[Y|X=x]$ as $\sum_{i=1}^n w(X_i, x) Y_i$ where
\begin{equation} \label{eq:rf_weight}
    w(X_i,x) = \frac{1}{m} \sum_{t=1}^m w(X_i, x; \theta_t).
\end{equation}
From now on, we call the weight $w(X_i,x)$ in \eqref{eq:rf_weight} as \textit{random forest weight}. The above representation of the random forest prediction of the mean can be equivalently obtained as a solution to the least-squares optimization problem $\min_{\lambda \in \mathbb{R}} \sum_{i=1}^n w(X_i,x) (Y_i - \lambda)^2$. Therefore, a least-squares regression adjustment, as the above, is equivalent to  \cite{breiman2001random} representation of random forests. However, when we move to estimation quantities that are not the mean, the latter representation is very powerful. Namely, a quantile random forest of \cite{meinshausen2006quantile} can be seen as a quantile regression adjustment \citep{li2017forest}, i.e., as a solution to the following optimization problem
\[
\min_{\lambda \in \mathbb{R}} \sum_{i=1}^n w(X_i, x) \rho_{\tau}(Y_i - \lambda),
\]
where $\rho_\tau$ is the $\tau$-th quantile loss function, defined as $\rho_{\tau}(u) = u(\tau - \ind(u < 0))$.
Local linear regression adjustment was also recently utilized in \cite{athey2019generalized} to obtain a smoother and more poweful generalized forest algorithm.

\paragraph{Generalized Random Forests}
\label{subsec:grf}

\cite{athey2019generalized} proposed to generalize random forest using a more sophisticated splitting criterion which is model-free. The new criterion aims to maximize the in-sample heterogeneity, formally defined as
\begin{align*}
    \tilde{\Delta}(C_1, C_2) = \sum_{j=1}^2 \frac{1}{|\{i:X_i \in C_j\}|} \left( \sum_{\{i:X_i \in C_j\}} \rho_i \right)^2
\end{align*}
where $\rho_i$ is a pseudo-response defined similarly as in Gradient Boosting \citep{friedman2001greedy}. Note that in the original random forest, we simply have $\rho_i = Y_i - \bar{Y}_P$ where $\bar{Y}_P$ is the mean response in the parent node. The generalized random forest, while applied to quantile regression problem, can deal with \textit{heteroscedasticity} because the splitting rule directly targets changes in the quantiles of the $Y$-distribution.

Just like the random forest algorithm, the generalized random forest is also an ensemble of trees and hence defines a weight or similarity between two samples using \eqref{eq:rf_weight}. The main difference hence lies in how they split the samples into different terminal regions. Therefore, in the following sections, whenever we refer to \textit{forest weight}, it can be calculated from either random forest or generalized random forest. In the experiment section, we will distinguish them by \textit{RF-weights} and \textit{GRF-weights}.

\section{Censored Quantile Regression Forest}
\label{sec:meth}
The forest regressions cannot be directly applied to censored data $\{(X_i, Y_i)\}$ because the conditional quantile of $Y$ is different from the quantiles of the latent variable $T$ due to the censoring. Moreover, there is no explicitly defined quantile loss function for randomly censored data. In this section, we design a new approach to achieve both tasks. We will motivate and derive our method using sections \ref{subsec:meth1} to \ref{subsec:meth2}.

\subsection{No Censoring and Locally Invariate $T_i$}
\label{subsec:meth1}
\paragraph{Assumptions:} We first assume \textbf{1}. there is no censoring on the data, and \textbf{2}. the latent variable $T_i$ has the same conditional probability in a neighborhood $R_x$ of $x$.

Following the regression adjustment reasoning, we could estimate the $\tau$-th quantile of $T_i$ at $x$ as
\[
q_{\tau, x} = \argmin_{q \in \mathbb{R}} \sum_{i=1}^n w(X_i, x) \rho_{\tau}(T_i - q).
\]
The above optimization problem has the following estimating equation
\begin{align} \label{eq:fop_latent}
U_n(q; x) &=(1-\tau) - \sum_{i=1}^n w(X_i, x) \ind(T_i > q) \approx 0.
\end{align}

Now out of the $n$ data points, assume $\{X_1, \cdots, X_k\} \subset R_x$ and $w(X_i, x) = k^{-1} \ind \{X_i \in R_x\}$. Because of the assumption 1, the estimating equation becomes
\begin{eqnarray} \label{eq:fop_approx}
U_k(q) &= (1-\tau) - \frac{1}{k} \sum_{i=1}^k \ind (T_i > q) \approx 0.
\end{eqnarray}
Now conditional on $\{x\} \cup \{X_i\}_{i=1}^k$, we have the expected estimating equation
\begin{equation*}
\mathbb{E} \left[ U_k(q) | x, X_i, i=1,\ldots,k \right] = (1-\tau) - \mathbb{P}(T > q | x)
\end{equation*}
which will be zero at $q^*$ where $\mathbb{P}(T > q^* | x) = 1-\tau$, that is, at the true $\tau$th quantile at $x$.

\subsection{With Censoring and Locally Invariate $T_i$}
Let's now consider the case that we could only observe $Y_i = \min\{T_i, C_i\}$ and the censoring indicator $\delta_i = \ind (T_i \le C_i)$. Note that the following analysis extends straightforwardly to left censoring. In order to estimate $q_{\tau,x}$, we cannot simply replace $T_i$ with $Y_i$ in \eqref{eq:fop_approx} as the $\tau$-th quantile of $T_i$ is no longer the $\tau$-th quantile of $Y_i$. However, because $C_i$ and $T_i$ are conditionally independent, we have the following relation:
\begin{align*}
\mathbb{P}(Y_i > q_{\tau,x}|x) &= \mathbb{P}(T_i > q_{\tau,x}|x) \mathbb{P}(C_i > q_{\tau,x}|x) \\
& = (1-\tau) G(q_{\tau,x}|x)
\end{align*}
where $G(u|x)$ is the conditional survival function of $C_i$ at $x$. That is to say, the $\tau$-th quantile of $T_i$ is actually the $1 - (1-\tau)G(q_{\tau,x}|x)$-th quantile of $Y_i$ at $x$. Now, if we define a new estimating equation that resembles  \eqref{eq:fop_approx} as follows
\begin{equation} \label{eq:fop_Y_approx}
S^o_k(q; x) = (1-\tau)G(q|x) - \frac{1}{k}\sum_{i=1}^k  \ind (Y_i > q)  \approx 0,
\end{equation}
we still have $\mathbb{E}[S_k^o(q_{\tau,x})|x] = 0$. An intuitive explanation for using \eqref{eq:fop_Y_approx} is that because  the $\tau$-th quantile of $T_i$ is just the $1 - (1-\tau)G(q_{\tau,x}|x)$-th quantile of $Y_i$ at $x$, instead of estimating the former which is not available because of the censoring, we could just estimate the later one. 

The survival function $G(\cdot|x)$ can be estimated by any consistent estimator, for example, the Kaplan-Meier estimator $\hat{G}(\cdot|x)$ using $\{Y_i\}_{i=1}^k$ and $\{\delta_i\}_{i=1}^k$, and we  can then solve for 
\begin{equation} \label{eq:fop_Y_approx_KM}
S_k(q; x) =  (1-\tau)\hat{G}(q|x) - \frac{1}{k}\sum_{i=1}^k \ind (Y_i > q) \approx 0.
\end{equation}

\subsection{Full Model}
\label{subsec:meth2}
In the previous section, we assume that $\mathbb{P}(T|X) = \mathbb{P}(T|x)$ for all $X \in R_x$. But in reality, this assumption is not always true, and that is why $w(X_i, x)$ plays an important rule in our final estimator, as it ``corrects" the empirical probability of each $T_i$ at $x$.
Intuitively, if $X_i$ is more similar to $x$ than $X_j$, then $Y_i$ should play a more important role than $Y_j$ on estimating the quantile at $x$. 
Now let $w(X_i,x)$ denote a similarity measure between $X_i$ and $x$. In order for $\sum_{i=1}^n w(X_i, x) \ind(T_i \le q)$ to be a proper estimation of $\mathbb{P}(T\le q|x)$, it needs to satisfy two conditions:
\begin{align*} \label{eq:conditions}
&(1) \; \sum_{i=1}^n w(X_i, x) = 1; \\
&(2) \; \sum_{i=1}^n w(X_i, x) \ind(T_i \le q) \overset{p}{\to} \mathbb{P}(T \le q | x) \; \forall q.
\end{align*}
One may think that any fixed Kernel weights, $K(X_i, x)$, could be a suitable choice, but in fact they would not be able to satisfy the second condition for every distribution $\mathbb{P}(T | x)$. Fortunately, as shown in \cite{meinshausen2006quantile} and \cite{athey2019generalized}, the data-adaptive (generalized) random forest weight $w(X_i, x)$ perfectly satisfies both conditions. Therefore if we define
\begin{equation}
U_n(q_{\tau,x}) = (1-\tau) - \sum_{i=1}^n w(X_i, x) \ind(T_i > q_{\tau,x}),
\end{equation}
we have $U_n(q_{\tau,x}) \overset{p}{\to} 0$ asymptotically. Then following the same logic of how we get \eqref{eq:fop_Y_approx_KM}, a general case estimating equation for censoring data will be
\begin{align} \label{eq:final_fop}
S_n(q; x) = (1-\tau)\hat{G}(q|x) - \sum_{i=1}^n w(X_i, x)\ind (Y_i > q) \approx 0.
\end{align}

\subsection{Estimators for $G(q|x)$}
\label{subsec:meth3}

Many consistent estimators for the conditional survival functions exist. For example, the nonparametric estimator \citep{beran1981nonparametric}
\begin{equation} \label{eq:Beran}
\tilde{G}(q|x) = \prod_{Y_i \le q} \left\{ 1 - \frac{W_i(x, a_n)}{\sum_{j=1}^n \ind(Y_j \ge Y_i) W_j(x,a_n)} \right\}^{1-\delta_i}
\end{equation}
is shown to be consistent \citep{beran1981nonparametric,dabrowska1987non,dabrowska1989uniform,gonzalez1994asymptotic,akritas1994nearest,li1995approach,van1996uniform}. Here, $W_i(x,a_n)$ is the Nadaraya-Watson weight.
However, since we already have an adaptive version of kernel -- the forest weights $w(X_i, x)$, we propose the following two  new estimators for $G(q|x)$:

\paragraph{Kaplan-Meier using nearest neighbors.}
We first find the $k$ nearest neighbors of $x$ according to the magnitude of the weights $w(X_i, x)$, and denote these points as a set $N_x$. Then we define the Kaplan-Meier estimator on $N_x$ as
\begin{align} \label{eq:KM_kNN}
\prod_{i:X_i \in N_x, Y_i\le q} \left( 1 - \frac{1}{\sum_{j=1}^n \ind(Y_j \ge Y_i) \ind(X_j \in N_x)} \right)^{1-\delta_i}.
\end{align}
Here, the number of nearest neighbors $k$ will be a tuning hyperparameter.

\paragraph{Beran estimator with forest weights.}
We replace the Nadaraya-Watson weights in \eqref{eq:Beran} with the forest weights $w(X_i, x)$:
\begin{equation} \label{eq:Beran_rf}
\hat{G}(q|x) = \prod_{Y_i \le q} \left\{ 1 - \frac{w(X_i, x)}{\sum_{j=1}^n \ind(Y_j \ge Y_i) w(X_j, x)} \right\}^{1-\delta_i}.
\end{equation}
In fact \eqref{eq:KM_kNN} is a special case of \eqref{eq:Beran_rf} when the weight $w(X_i, x) = 1/k$ for $X_i \in R_x$ and $0$ otherwise.

\subsection{Algorithm}
We summarize our algorithm in Algorithm \ref{alg:main}. The details for choosing the candidate set $\mathcal{C}$ is in Section \ref{subsec:alg}.

\begin{algorithm}
\caption{Censored quantile regression forest}
\label{alg:main}
\begin{algorithmic}[2]

\STATE{\textbf{Input:} number of trees: $B$, minimum node size: $m$, the number of nearest neighbors: $k$ (if using \eqref{eq:KM_kNN}), test set $\mathcal{A}$, training set $ \mathcal{D}=\{(X_i, Y_i, \delta_i)\}_{i=1}^n$, quantile $\tau$}

\STATE \textbf{Step 0:} Train a forest on $(\mathcal{D})$ with $B$ trees and minimum node size $m$.

\FOR{$x \in \mathcal{A}$}
\STATE \textbf{Step 1:} Calculate forest weights $w(x, X_i$).

\STATE \textbf{Step 2:} Calculate the survival function estimate $\hat{G}(q|x)$.

\STATE \textbf{Step 3:} Get the quantile estimation:
$$\hat q(x) \gets \argmin_{q \in \mathcal{C}} |S_n(q; x)|$$

\COMMENT{$\mathcal{C}$ is a candidate set as discussed in Section \ref{subsec:alg}}
\COMMENT{$S_n$ is defined in \eqref{eq:final_fop}} \label{step:4}
\ENDFOR

\end{algorithmic}
\end{algorithm}
\section{Theoretical Develoments}
\label{sec:theory}

In this section, we will show the consistency of the proposed quantile estimator. The time complexity analysis is in the Appendix.


\subsection{Consistency}
\label{subsec:consistency}
The consistency of random forest has been extensively studied \citep{arlot2014analysis,athey2019generalized,biau2008consistency,biau2010layered,biau2012analysis,denil2014narrowing,lin2006random,scornet2015consistency,wager2015adaptive,wager2018estimation}. Following the common settings, we also assume the covariate space $\mathcal{X} = [0,1]^p$ and the parameter $q \in \mathcal{B} \subset \mathbb{R}$ where $\mathcal{B}$ is a compact subset of $\mathbb{R}$. In our case, since $q$ stands for the quantile, the assumption means that there exists some $r > 0$ such that $q \in [-r, r] = \mathcal{B}$. We also make another standard assumption that the density of $X$ is bounded away from $0$ and $\infty$. Note that since $\mathcal{X}$ is a compact support, the density condition holds true for Gaussian distribution and more broadly any symmetric and continuous distribution with unbounded support.

\begin{condition}[Lipschitz in $x$]\label{cond:3}
Denote $F(y|x) = \mathbb{P}(Y \le y | x)$. There exists a constant $L$ such that $F(y|x)$ is Lipschitz continuous with parameter $L$, that is, for all $x, x^{'} \in \mathcal{X}$,
\begin{equation*}
\sup_{y} |F(y|x) - F(y|x^{'})| \le L \|x - x^{'}\|_1.
\end{equation*}
\end{condition}

This Condition \ref{cond:3} appears in all existing work related to quantile regression and inference thereafter.  

\begin{condition}[Identification]\label{cond:4}
For any fixed $x$, the latent variable $T$ and the censoring variable $C$ are conditionally independent, and the conditional distribution $\mathbb{P}(T \le q|x)$ and $\mathbb{P}(C \le q|x)$ are both strictly increasing in $q$.
\end{condition}

Conditional independence of $T$ and $C$ is a very standard assumption and can be traced back to \cite{robins1992semiparametric} among other works.

\begin{condition}[Tree splitting]\label{cond:2}
For each tree splitting, the probability that each variable is chosen for the split point is bounded from below by a positive constant, and every child node contains at least $\gamma$ proportion of the data in the parent node, for some $\gamma \in (0, 0.5]$. [Quantile forest \citep{meinshausen2006quantile}] The terminal node size $m \to \infty$ and $m/n \to 0$ as $n \to \infty$. [Generalized forest \citep{athey2019generalized}] The forest is honest and built via subsampling with subsample size $s$ satisfying $s/n \to 0$ and $s \to \infty$.
\end{condition}

The first two requirements of Condition \ref{cond:2} are shared in \cite{meinshausen2006quantile} and \cite{athey2019generalized}. For quantile random forest \citep{meinshausen2006quantile}, they require that the leaf node size of each tree should increase with the sample size $n$, but at a slower rate. Our experiments also justify that the required leaf node size of Meinshausen's quantile forest is larger than the node size of the generalized forest. In general, using the generalized forest weights give us more stable estimations because the trees are honest and regular \citep{wager2018estimation}.

\begin{condition}[Censoring variable]\label{cond:5}
For any $x \in \mathcal{X}$, $\hat{G}(q|x)$ is a uniformly consistent estimator of the true conditional survival function $G(q|x)$ for $q \in \mathcal{B}$.
\end{condition}

Condition \ref{cond:5} is satisfied, for example, by the Kaplan-Meier estimator \eqref{eq:Beran} \citep{dabrowska1989uniform}. Please take a look at Figure \ref{fig:g_comparison} where we compare finite sample properties of the newly introduced estimators \eqref{eq:KM_kNN} and \eqref{eq:Beran_rf}. We observe that the new distributional estimators are more adaptive and yet seemingly inherit consistency to that of the traditional KM estimator.

We proceed to showcase asymptotic properties of the proposed estimating equations. We begin by illustrating a concentration of measure phenomenon for the introduced score equations.

\begin{theorem}
\label{thm:part1}
Define
\begin{equation}
S(q;\tau) = (1-\tau)G(q|x) - \mathbb{P}(Y > q).
\end{equation}
Under Conditions \ref{cond:3} -- \ref{cond:5}, for any $x \in \mathcal{X}$, $r > 0$ and $\tau \in (0,1)$, we have
\begin{equation*}
\sup_{q \in [-r,r]} | S_n(q; \tau) - S(q; \tau) | = o_p(1).
\end{equation*}
\end{theorem}

Next, we present our main result that illustrates an asymptotic consistency of the proposed conditional quantile estimator. The proof is given in Appendix.

\begin{theorem}
\label{thm:consistency}
Under Conditions \ref{cond:3} -- \ref{cond:5}, for fixed $\tau \in (0,1)$ and $x \in \mathcal{X}$, define $q^*$ to be the root of $S(q;\tau) = 0$, and $r > 0$ to be some constant so that $q^* \in [-r, r]$. Also define $q_n$ to be $\argmin_{q \in [-r, r]} \left|S_n(q; \tau)\right|$. Then $\mathbb{P}(T \le q^*|x) = \tau$, and $q_n \overset{p}{\to} q^*$ as $n \to \infty$.

\end{theorem}

\section{Experiments}
\label{sec:experiments}

In this section, we will compare the proposed model, censored regression forest (\textit{crf}), with generalized random forest (\textit{grf}) \citep{athey2016generalized}, quantile random forest (\textit{qrf}) \citep{meinshausen2006quantile} and random survival forest (\textit{rsf}) \citep{hothorn2005survival} on various simulated and real datasets. On the simulated datasets, we report both censored and oracle results for \textit{qrf} and \textit{grf}. To obtain the censored result, we directly apply generalized random forest and quantile random forest to the censored data, and denote the results by \textit{grf} and \textit{qrf} respectively. For oracle result, we instead train the models using the oracle responses without censoring (i.e. $T_i$'s), and call the results \textit{grf-oracle} and \textit{qrf-oracle}.

\subsection{Simulation Study}
\label{ssec:simulation}

In this section, we denote censored regression forest with generalized forest weights as \textit{crf-generalized} and the one with (quantile) random forest weights as \textit{crf-quantile}. We first define the evaluation metric used in this section -- \textbf{quantile loss}. The $\tau$-th quantile loss is defined as follows. Let $\hat{q}^{\tau}_i$ be the estimated $\tau$-th quantile at $X_i$, then
\begin{equation}\label{eq:L_quantile}
    L_{quantile}(\hat{q}^\tau_1, \ldots, \hat{q}^\tau_n) = \frac{1}{n} \sum_{i=1}^n \rho_{\tau}(T_i - \hat{q}^{\tau}_i).
\end{equation}
We could use this metric because we know the latent responses $T_i$'s in simulations.

\subsubsection{Accelerated Failure Time Data}
In this section, we generate data from a accelerated failure time (AFT) model. We sample $n=1000$ independent and identically distributed examples where $X_i$ is uniformly distributed over $[0, 2]^p$ with $p=20$, and $T_i$ is conditional on $(X_i)_1$ and $\log(T_i|X_i) = (X_i)_1 + \epsilon$ where $\epsilon \sim \mathcal{N}(0, 0.3^2)$. The censoring variable $C_i \sim \text{Exp}(\lambda=0.08)$ and $Y_i = \min(T_i, C_i)$. This results in about 23\% censoring level. The other 19 covariates are noise. We estimate the quantiles at $\tau=0.1$ and $0.9$, and draw the predicted quantiles in Figure \ref{fig:aft_20d} for node size $20$.

\begin{figure}
    \centering
    \vspace{-0.05in}
    \includegraphics[width=0.85\linewidth]{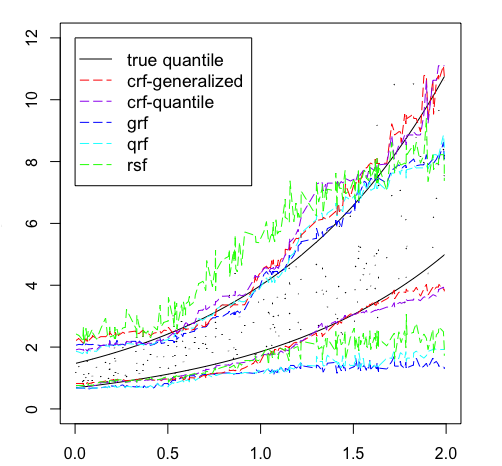}

    
    \caption{\label{fig:aft_20d}Quantile estimates on AFT model when $\tau = 0.1$ and $0.9$. The minimum node size is 20 and each forest contains 1000 trees.}
\end{figure}

From the results in Figure \ref{fig:aft_20d}, both \textit{grf} and \textit{qrf} are severely biased downwards because of the right censoring. The proposed methods \textit{crf-generalized} and \textit{crf-quantile} both provide consistent quantile estimation that is almost identical to the true quantiles. Random survival forest is very unstable on predicting the quantiles and is not able to correct the censoring bias. We increase the node size to from 20 to 80, all the methods except for \textit{rsf} become more biased and less variant, but \textit{rsf} is still volatile.

We then repeat the above experiment ten times for different node sizes ranging from 10 to 80, and report the average and standard deviation of the quantile losses in Figure \ref{fig:aft_nodesize}. We observe that the performances of both \textit{crf-generalized} and \textit{crf-quantile} are close to their corresponding oracles, and are much better than \textit{grf} or \textit{qrf} on censored data. \textit{crf-quantile} (\textit{qrf}) performs slightly better than \textit{crf-generalized} (\textit{grf}), implying that original quantile random forest can be more effective when the data is homoscedastic. The random survival forest \textit{rsf} behaves only slightly better than the biased \textit{grf} and \textit{qrf}, but is much worse than the proposed methods.

\subsubsection{Heteroscedastic Data}

\begin{figure}[!htb]
    \centering
    \includegraphics[width=0.9\linewidth]{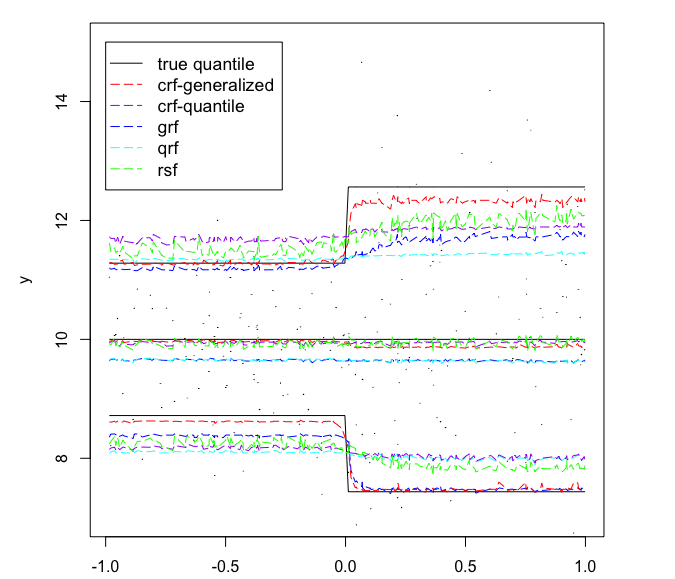}
    \caption{Quantile estimates on the censored heteroscedastic data for $\tau=0.1$, $0.5$ and $0.9$. The minimum node size for all trees is 150 and each forest contains 2000 trees.}
    \label{fig:hetero_1d}
\end{figure}

We test the proposed method on a heteroscedastic dataset. The dataset is taken from \citet{athey2019generalized}. We sample $n=2000$ independent and identically distributed examples where $X_i$ is uniformly distributed over $[-1, 1]^p$ with $p=40$, and $T_i$ is Gaussian conditionally on $(X_i)_1$ and $T_i|X_i \sim \mathcal{N}(10, (1 + \ind\{(X_i)_1 > 0\})^2)$. The censoring variable $C_i \sim 8 + \text{Exp}(\lambda=0.10)$ and $Y_i = \min(T_i, C_i)$. The other 39 covariates are noise. The censoring ratio is about 20\% in this example. We shift the mean of $T_i|X_i$ to 10 because the random survival forest only allows positive responses. We estimate the quantiles at $\tau=0.1, \, 0.5, \, 0.9$. The results are in Figure \ref{fig:hetero_1d}.

When the data is heteroscedastic, using generalized forest weights (\textit{crf-generalized}) provides much more accurate quantile estimation than all the other methods. The predicted quantiles by \textit{crf-generalized} are almost identical to the truths. Our results are inline with \citet{athey2019generalized} that generalized random forest is very effective at dealing with heteroscedasticity. Note that the random survival forest (\textit{rsf}) also fails to recognize the variance shift. This experiment indicates that our method coupled with generalized forest weights is the most, arguably the only effective method when dealing with heteroscedastic data.

We also repeat the above experiment ten times for different node sizes and report the quantile losses in Figure \ref{fig:hetero_nodesize}. We again observe that \textit{crf-generalized} achieves almost the same performance of \textit{grf-oracle}. \textit{rsf} and \textit{crf-quantile} have similar performance on this dataset, and are both slightly worse than even \textit{grf} when $\tau=0.1$. This shows that the splitting rule of random survival forest or quantile forest do not work well on heteroscedastic data.

\begin{figure*}[!h]
    \centering
    \begin{subfigure}[b]{0.3\textwidth}
        \includegraphics[width=\textwidth]{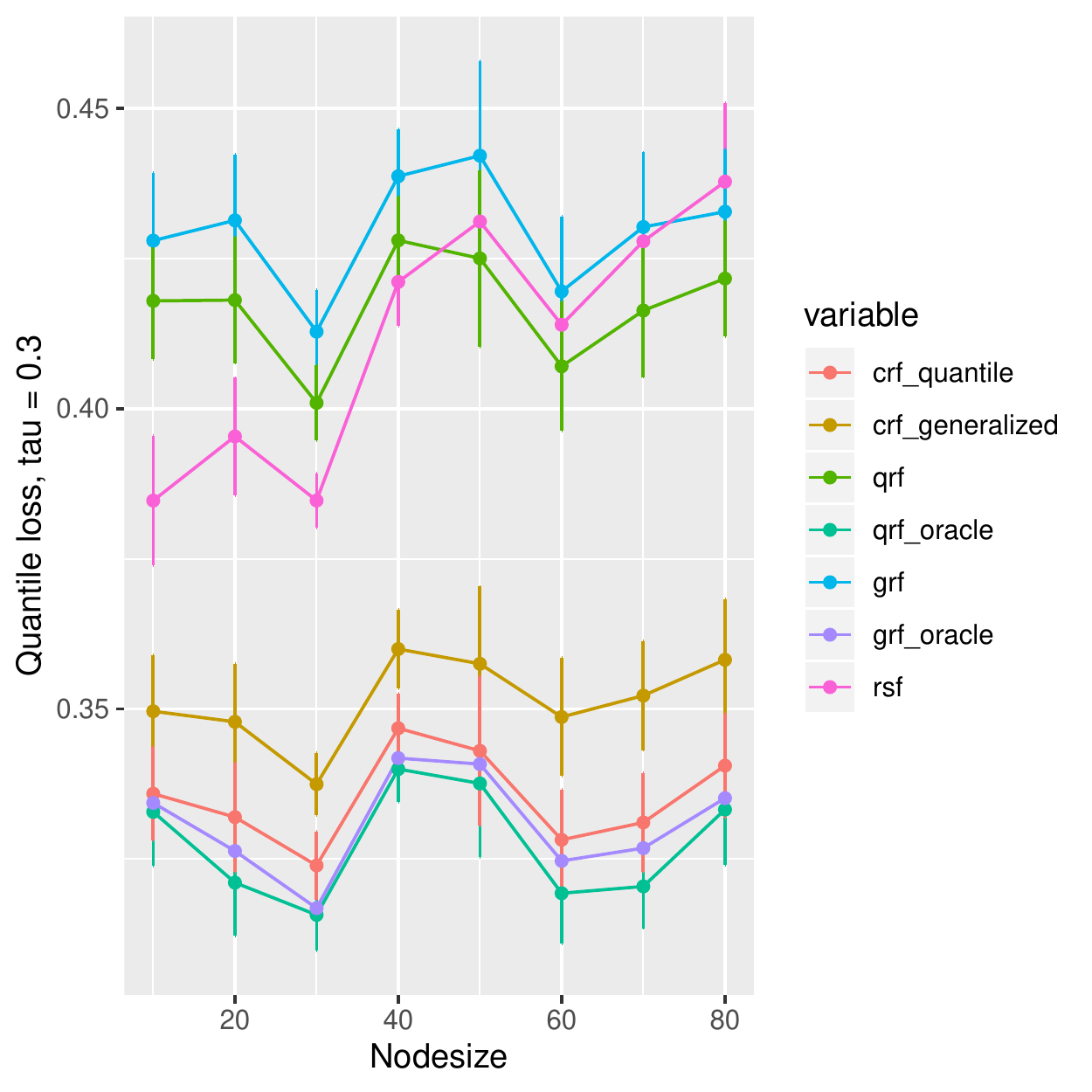}
        \caption{$\tau = 0.3$}
    \end{subfigure}
    ~
    \begin{subfigure}[b]{0.3\textwidth}
        \includegraphics[width=\textwidth]{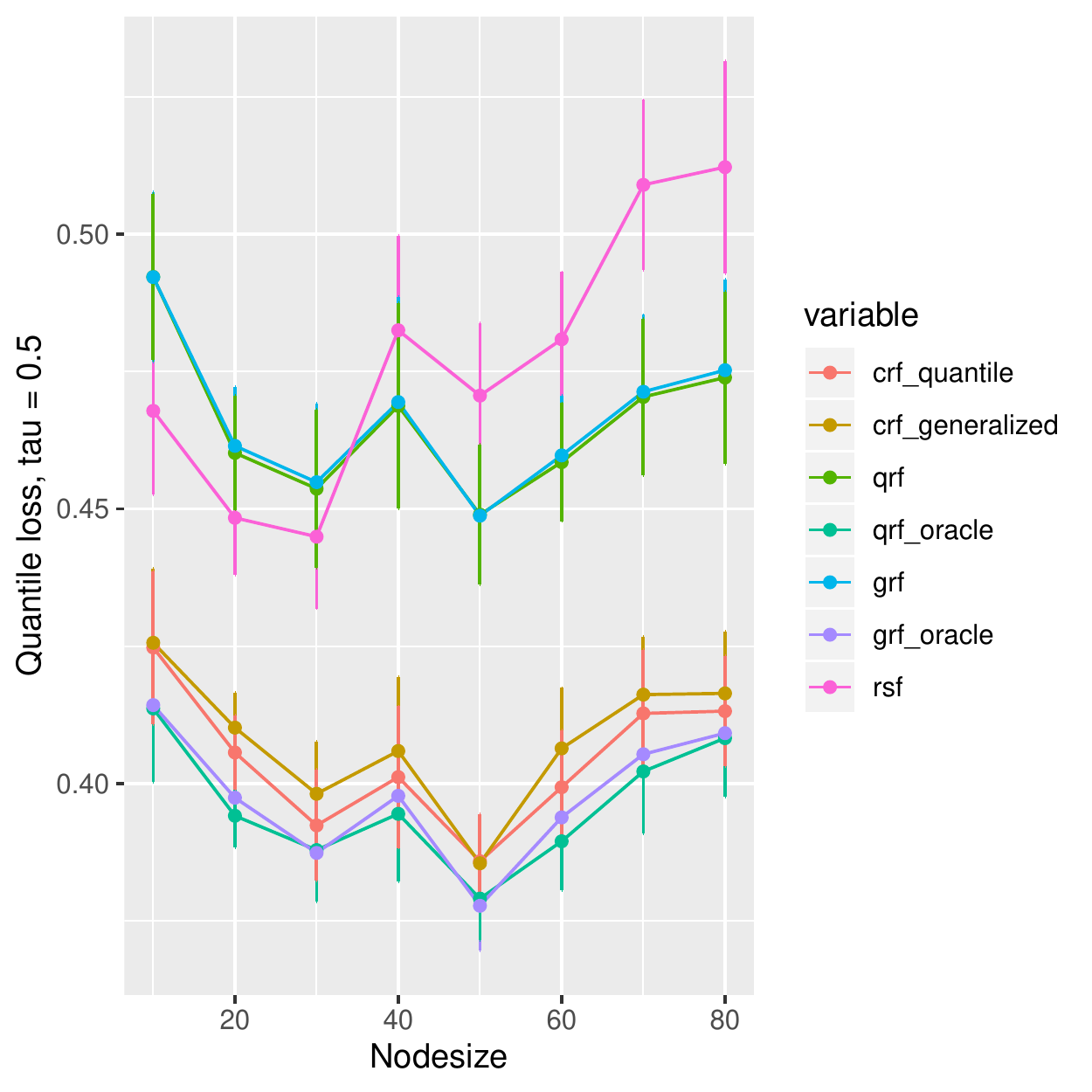}
        \caption{$\tau = 0.5$}
    \end{subfigure}
    ~
    \begin{subfigure}[b]{0.3\textwidth}
        \includegraphics[width=\textwidth]{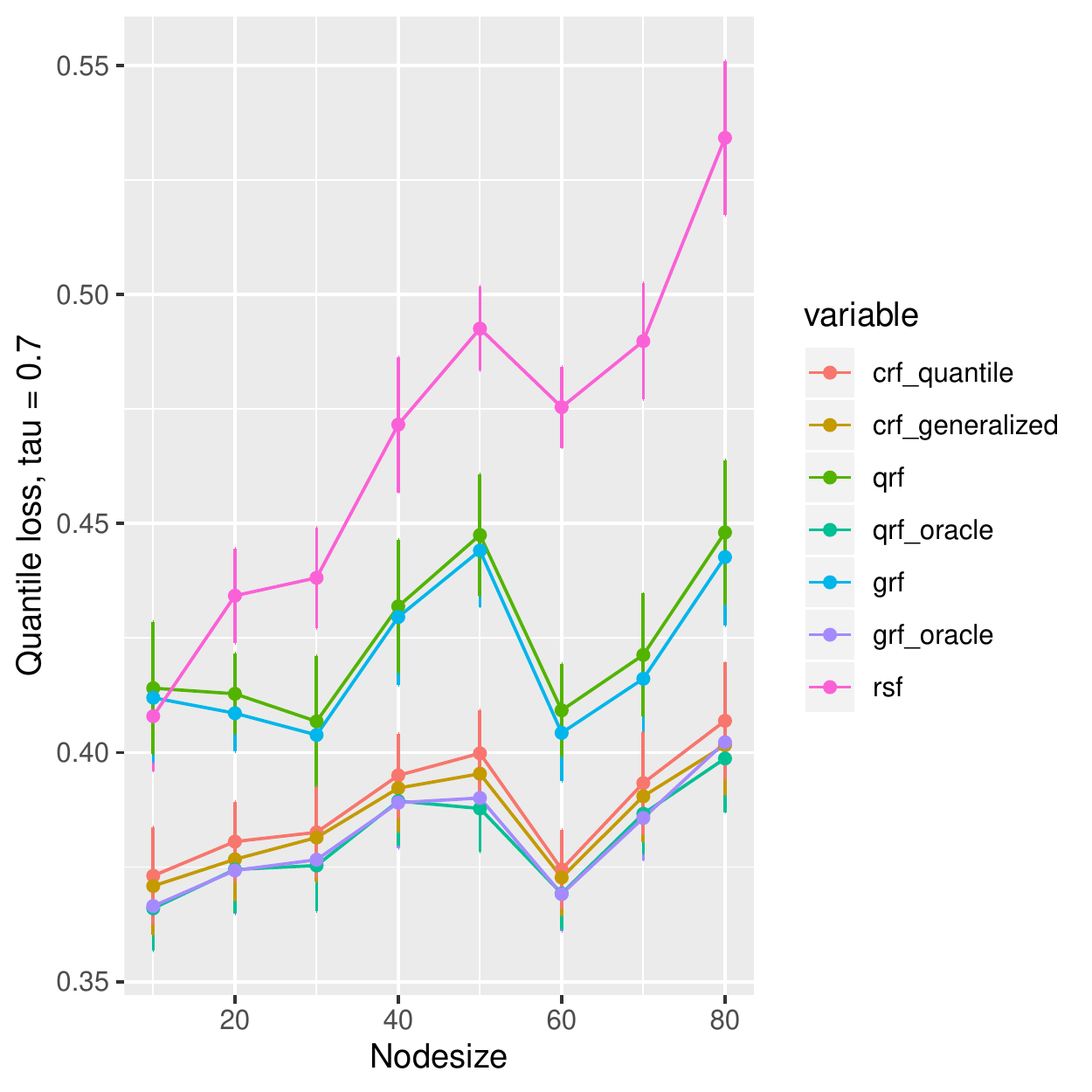}
        \caption{$\tau = 0.7$}
    \end{subfigure}
    
    \caption{Quantile losses on multi-dimensional AFT data with different node sizes.}
    \label{fig:aft_nodesize}
\end{figure*}

\begin{figure*}[!htb]
    \centering
    \begin{subfigure}[b]{0.3\textwidth}
        \includegraphics[width=\textwidth]{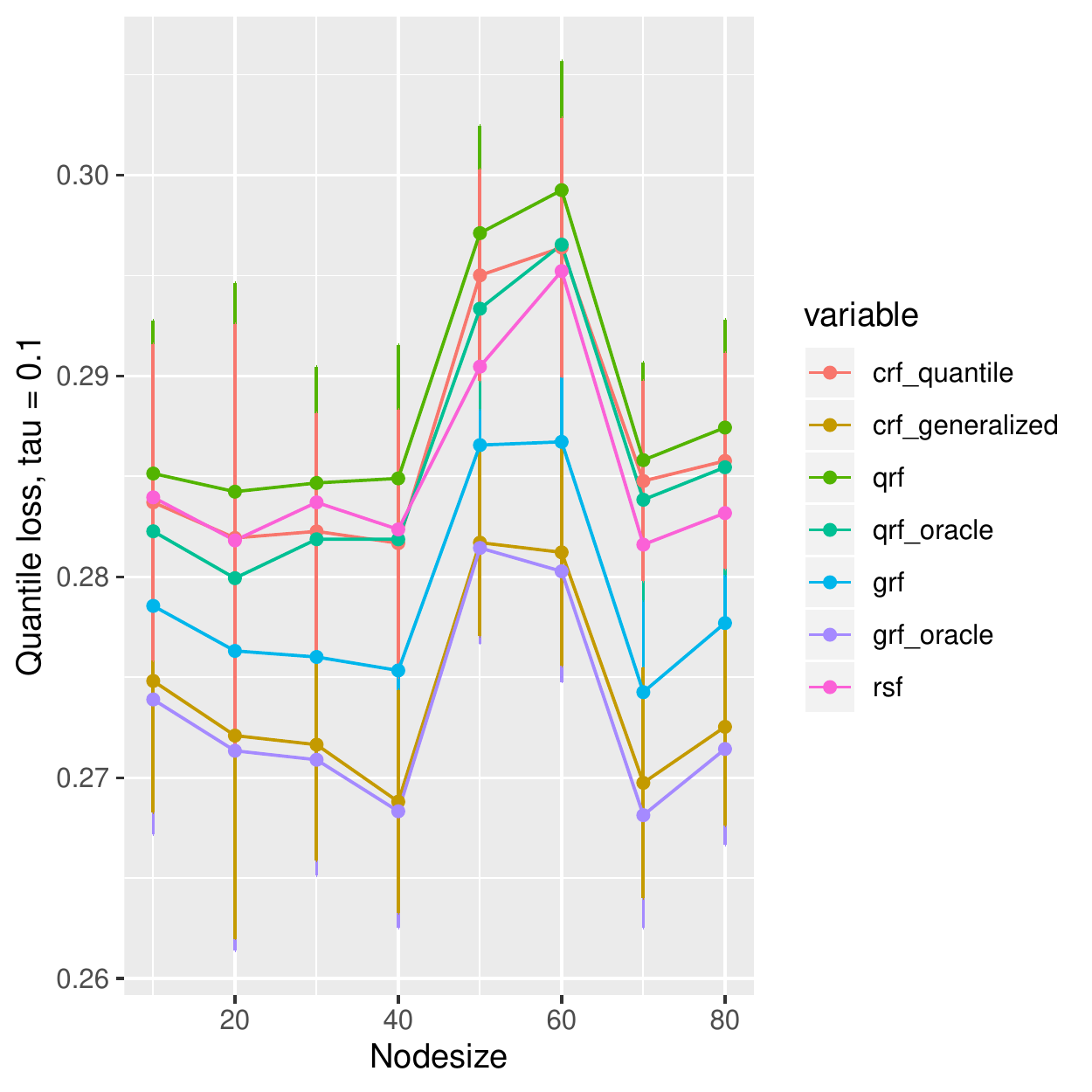}
        \caption{$\tau = 0.1$}
    \end{subfigure}
    ~
    \begin{subfigure}[b]{0.3\textwidth}
        \includegraphics[width=\textwidth]{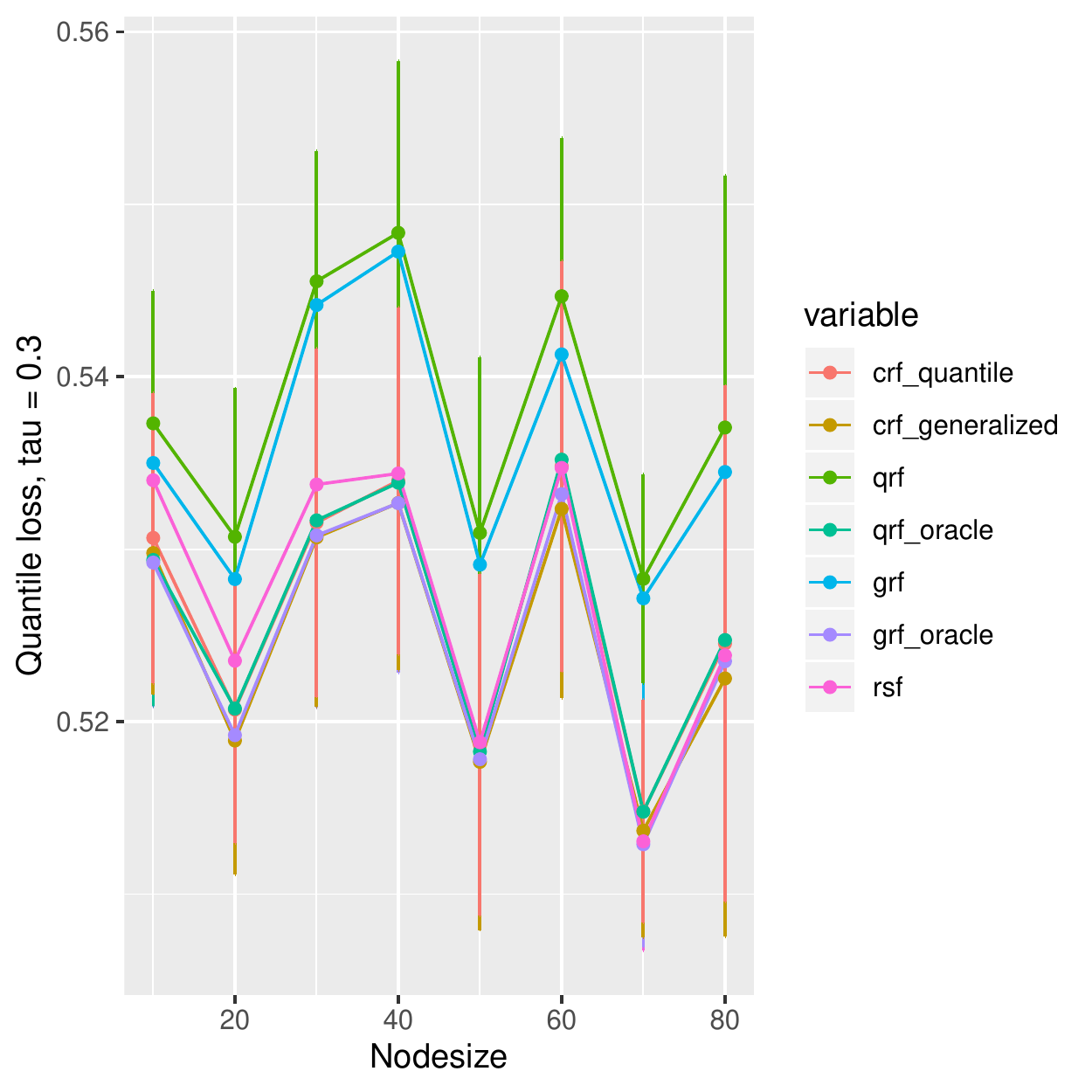}
        \caption{$\tau = 0.3$}
    \end{subfigure}
    ~
    \begin{subfigure}[b]{0.3\textwidth}
        \includegraphics[width=\textwidth]{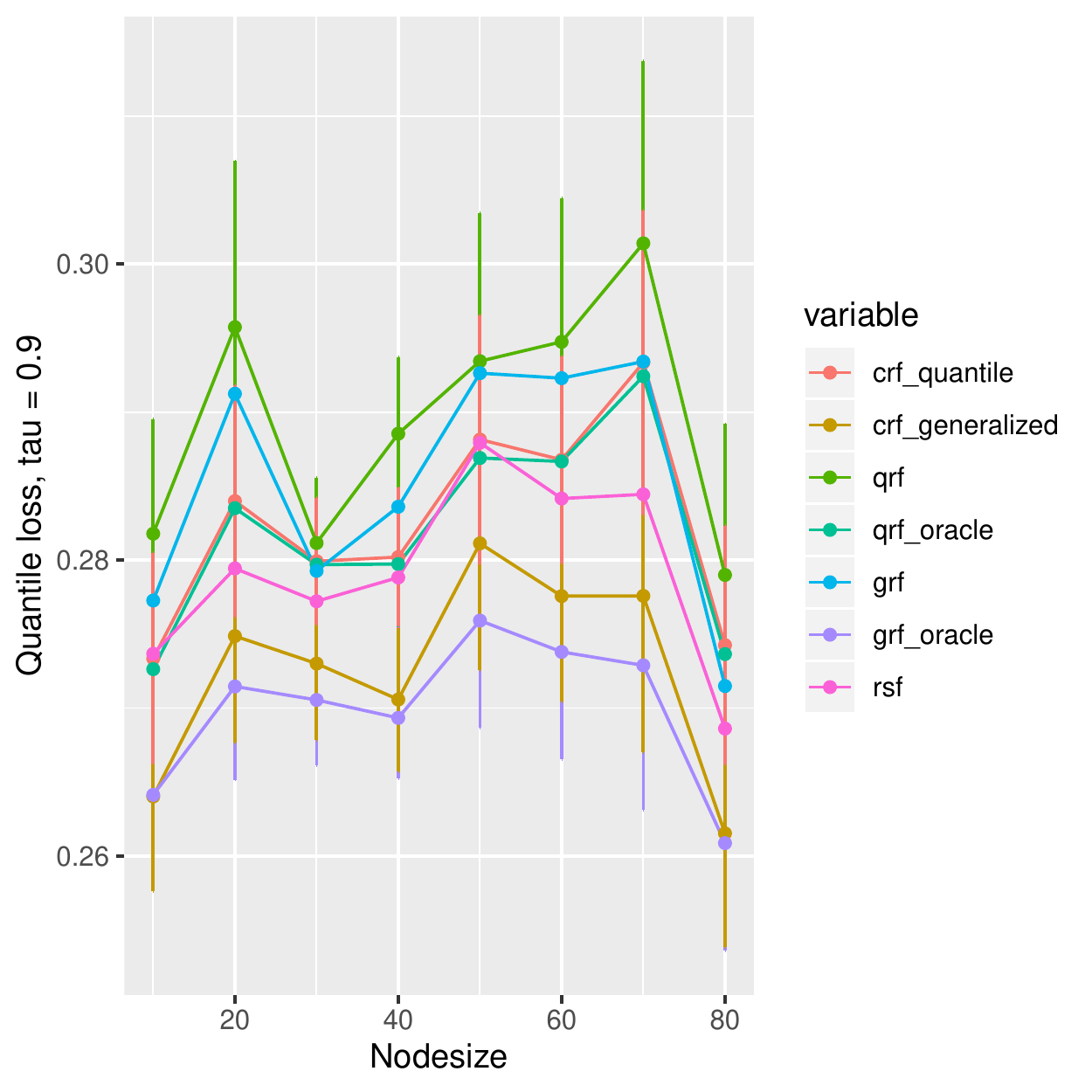}
        \caption{$\tau = 0.9$}
    \end{subfigure}
    
    \caption{Quantile losses on heteroscedastic examples.}
    \label{fig:hetero_nodesize}
\end{figure*}

\subsection{Conditional Survival Functions}

\begin{figure*}[!htb]
    \centering
    \vspace{-0.05in}
    \begin{subfigure}[b]{0.22\linewidth}
        \includegraphics[width=\textwidth]{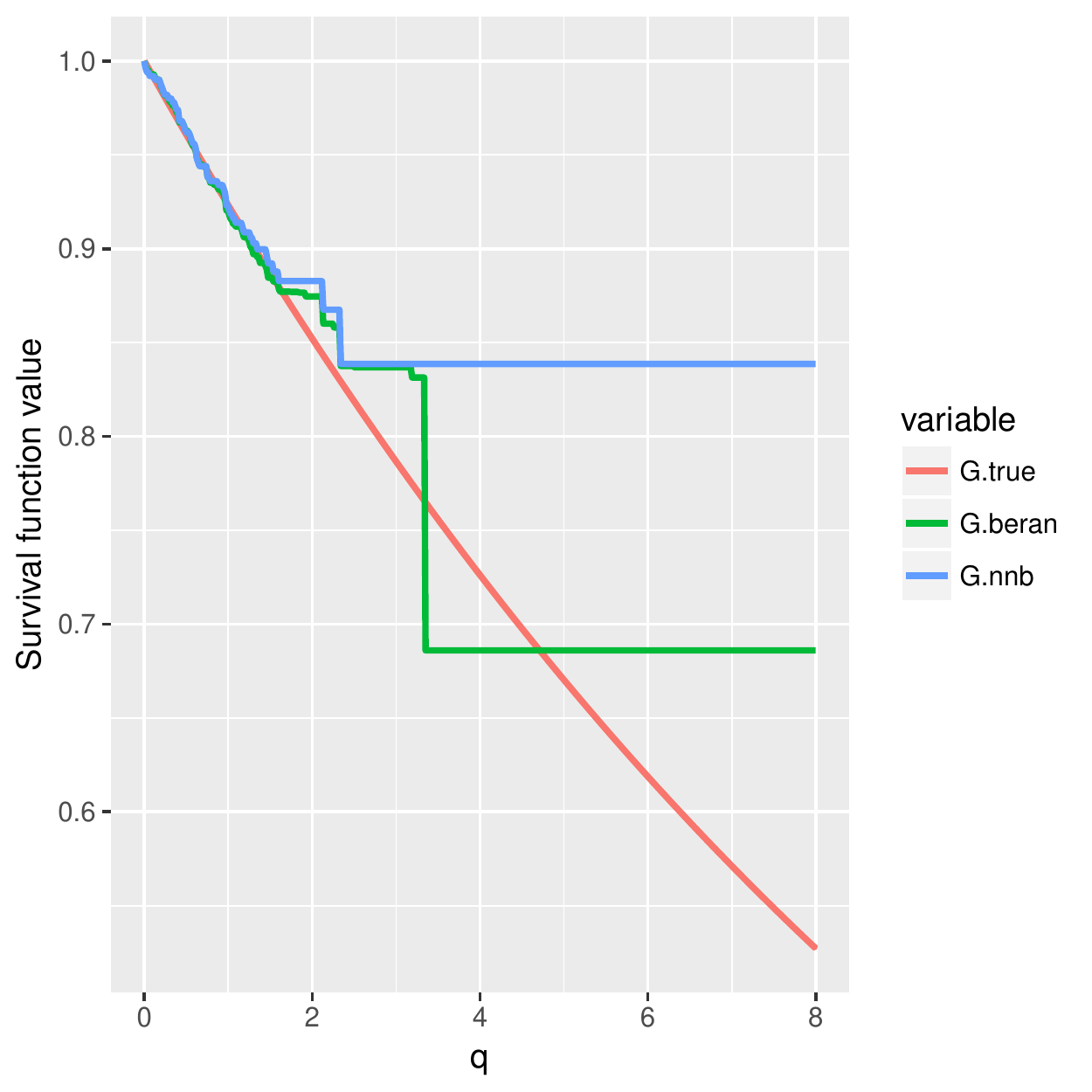}
        \caption{$x=0.4$}
    \end{subfigure}
    ~
    \begin{subfigure}[b]{0.22\linewidth}
        \includegraphics[width=\textwidth]{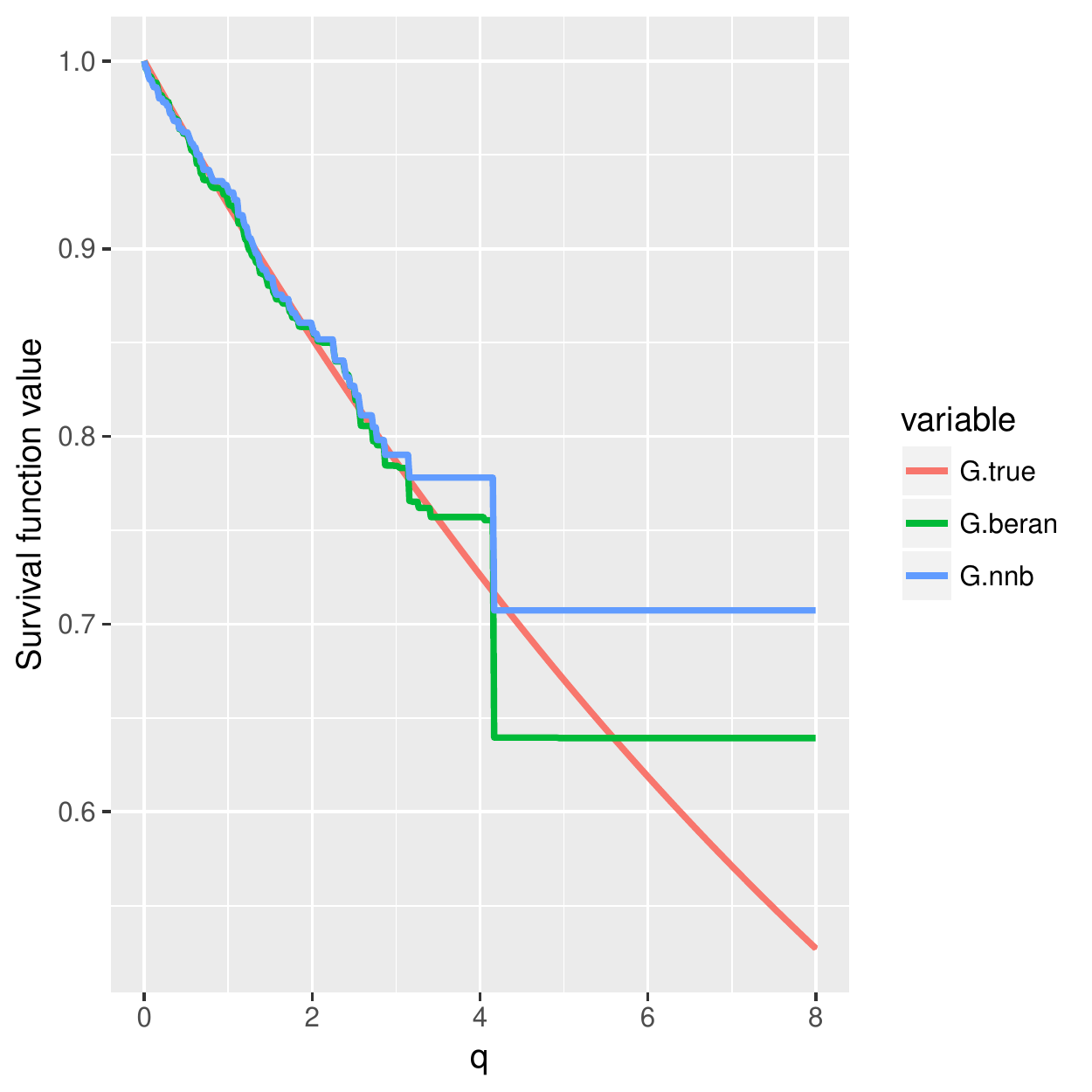}
        \caption{$x=0.8$}
    \end{subfigure}
    ~
    \begin{subfigure}[b]{0.22\linewidth}
        \includegraphics[width=\textwidth]{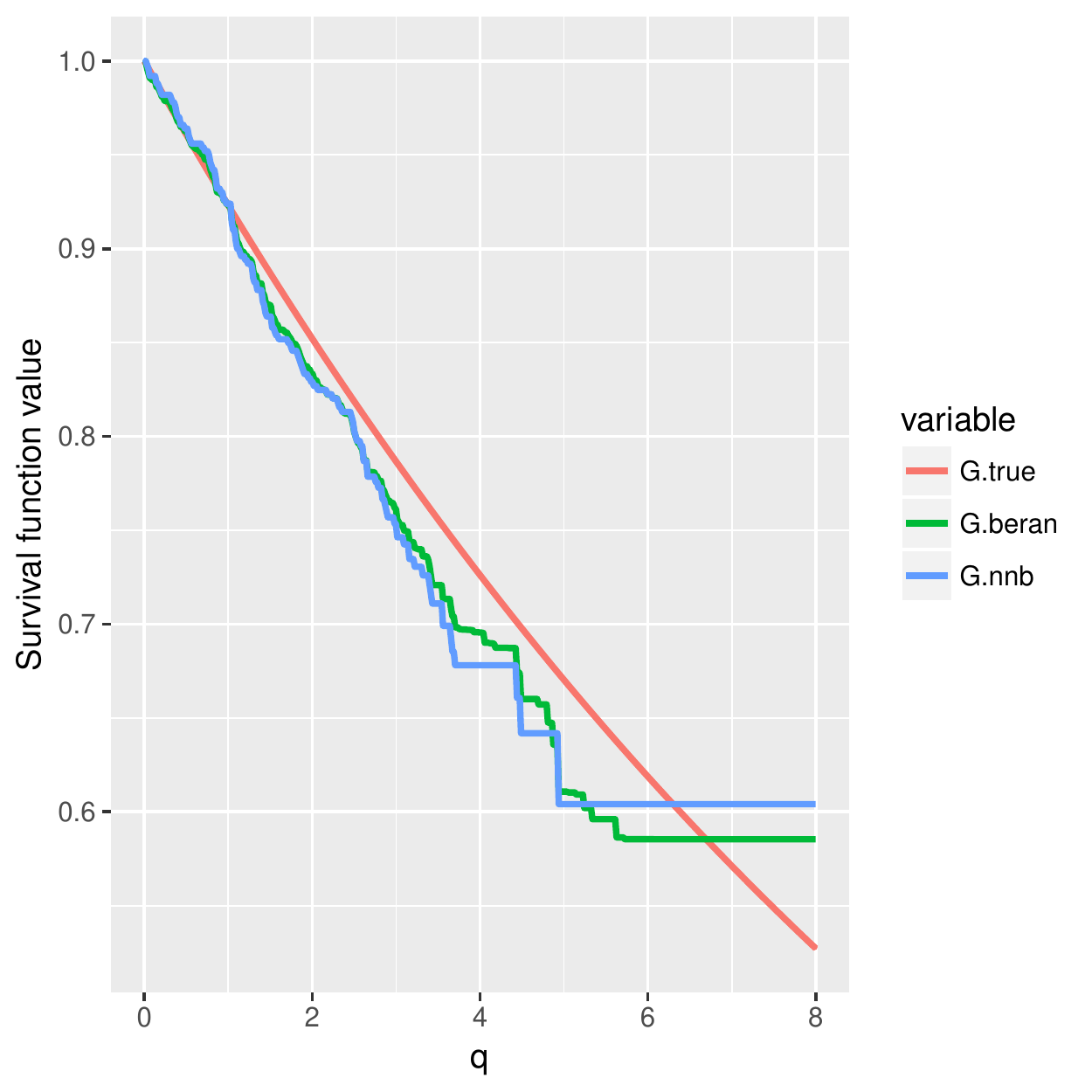}
        \caption{$x=1.2$}
    \end{subfigure}
    ~
    \begin{subfigure}[b]{0.22\linewidth}
        \includegraphics[width=\textwidth]{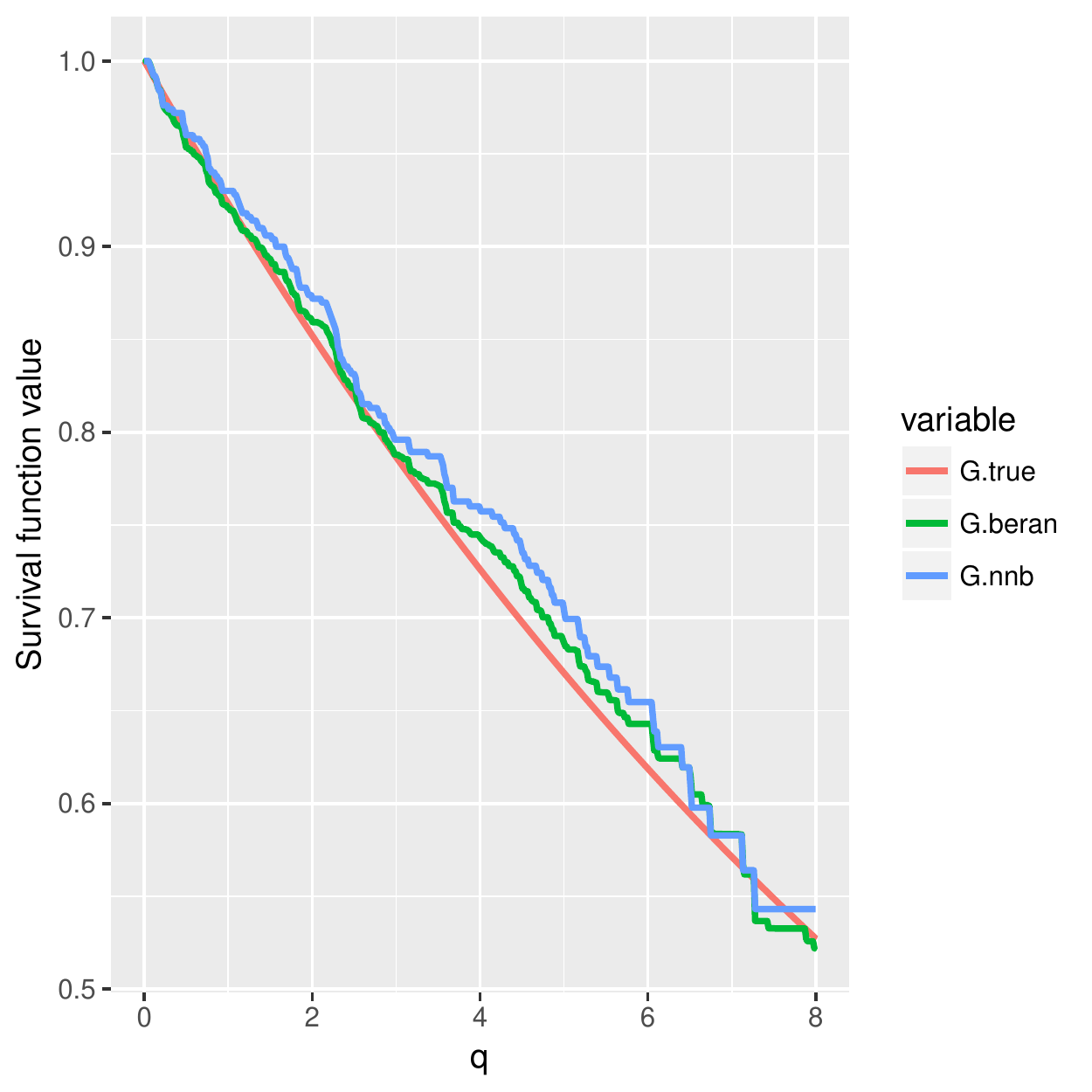}
        \caption{$x=1.6$}
    \end{subfigure}
    
    \caption{Comparison of the two conditional survival estimators on the AFT data. The sample size is 5000. For the nearest neighbor estimator \eqref{eq:KM_kNN}, we set the number of neighbors to be 10\% of the sample size.}
    \label{fig:g_comparison}
    \vspace{-0.05in}
\end{figure*}

In this section, we compare the two proposed conditional survival function estimators \eqref{eq:KM_kNN} and \eqref{eq:Beran_rf}. We generate examples from the AFT model, and then choose four test points $\{x_1=0.4, x_2=0.8, x_3=1.2, x_4=1.6\}$ to plot the conditional survival function estimations on them. The results are shown in Figure \ref{fig:g_comparison}.

\begin{figure*}[!htb]
    \small
    \centering
    \vspace{-0.1in}
    \begin{subfigure}[b]{0.3\linewidth}
        \includegraphics[width=\textwidth]{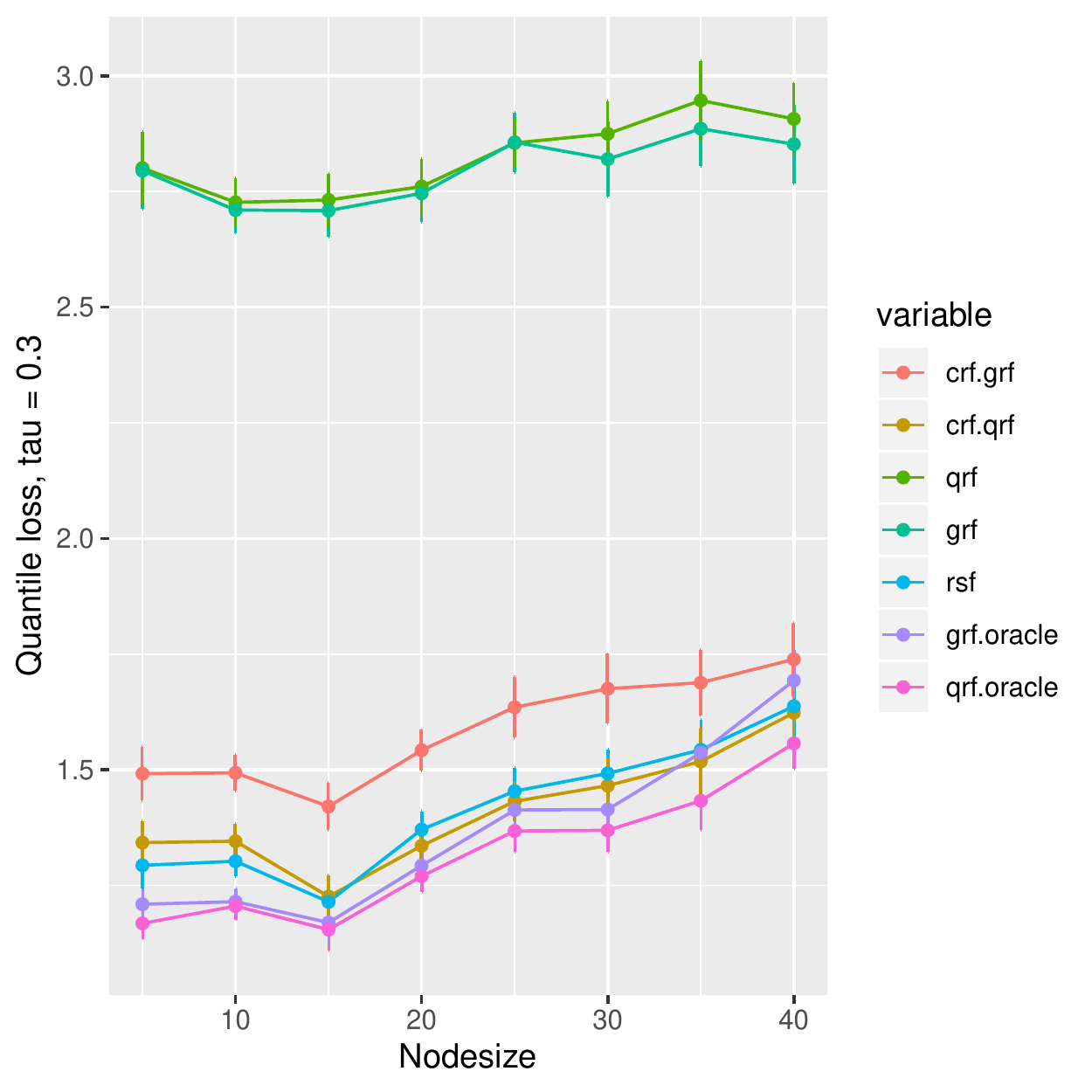}
        \caption{BostonHousing: $\tau = 0.3$}
    \end{subfigure}
    ~
    \begin{subfigure}[b]{0.3\linewidth}
        \includegraphics[width=\textwidth]{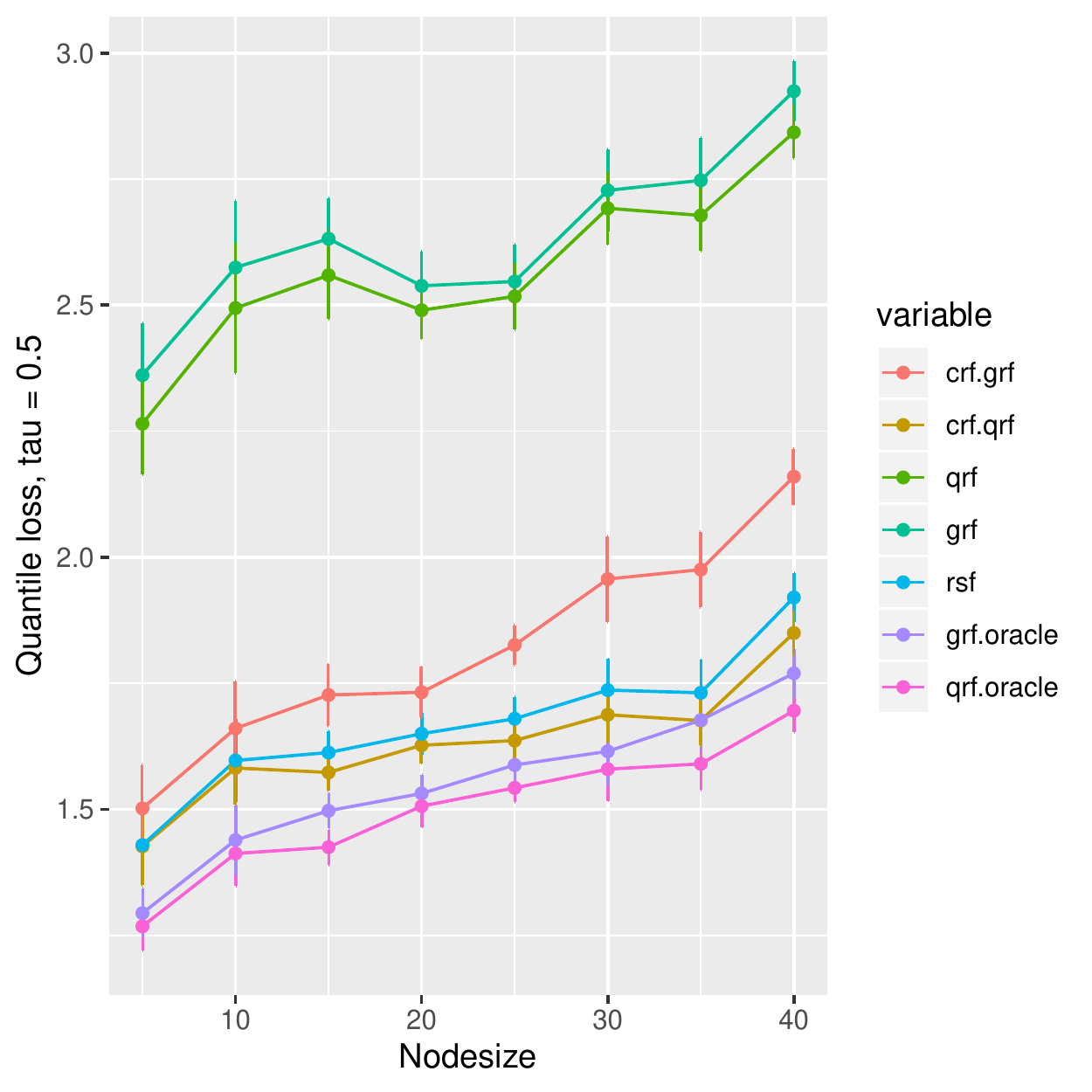}
        \caption{BostonHousing: $\tau = 0.5$}
    \end{subfigure}
    ~
    \begin{subfigure}[b]{0.3\linewidth}
        \includegraphics[width=\textwidth]{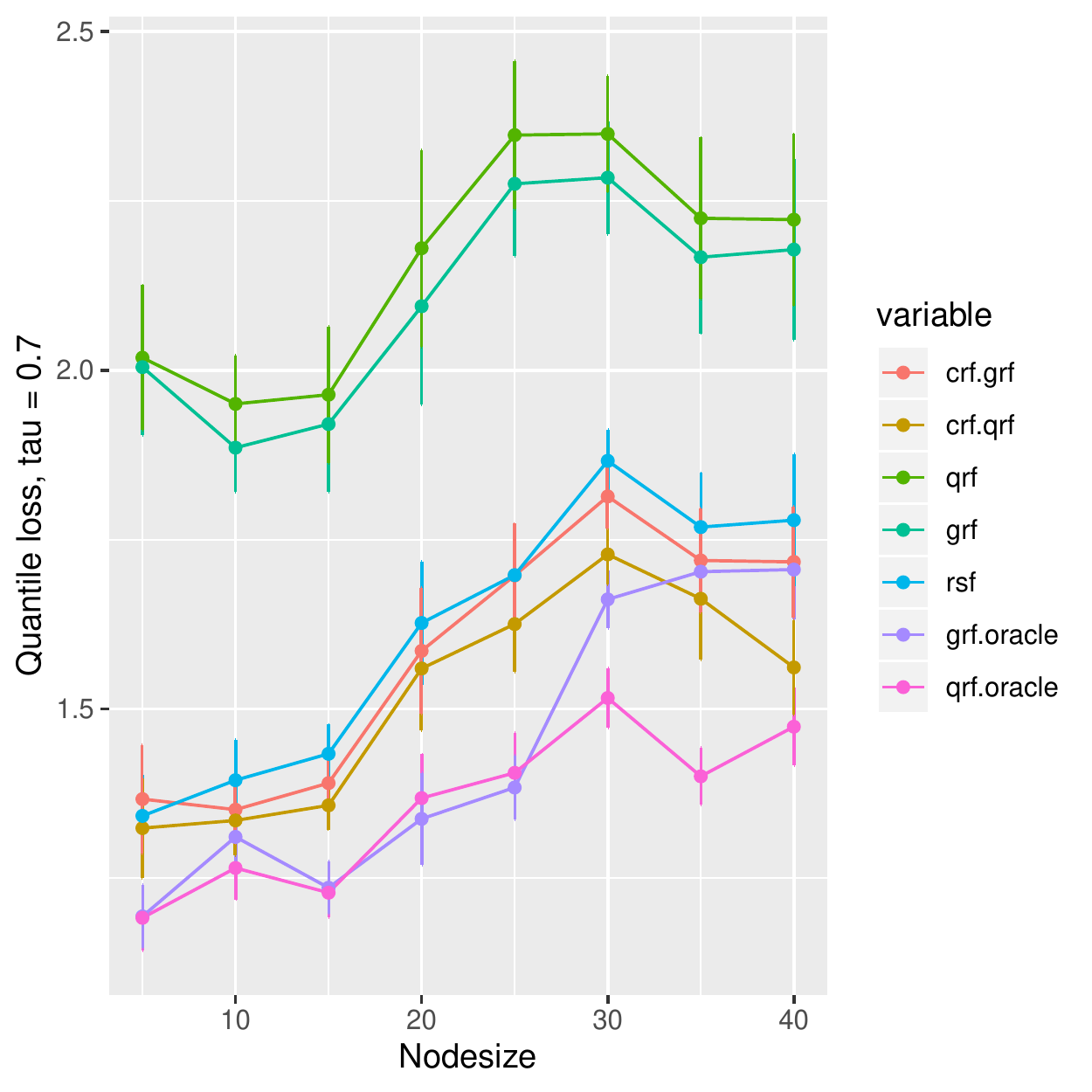}
        \caption{BostonHousing: $\tau = 0.7$}
    \end{subfigure}

    \begin{subfigure}[b]{0.3\linewidth}
        \includegraphics[width=\textwidth]{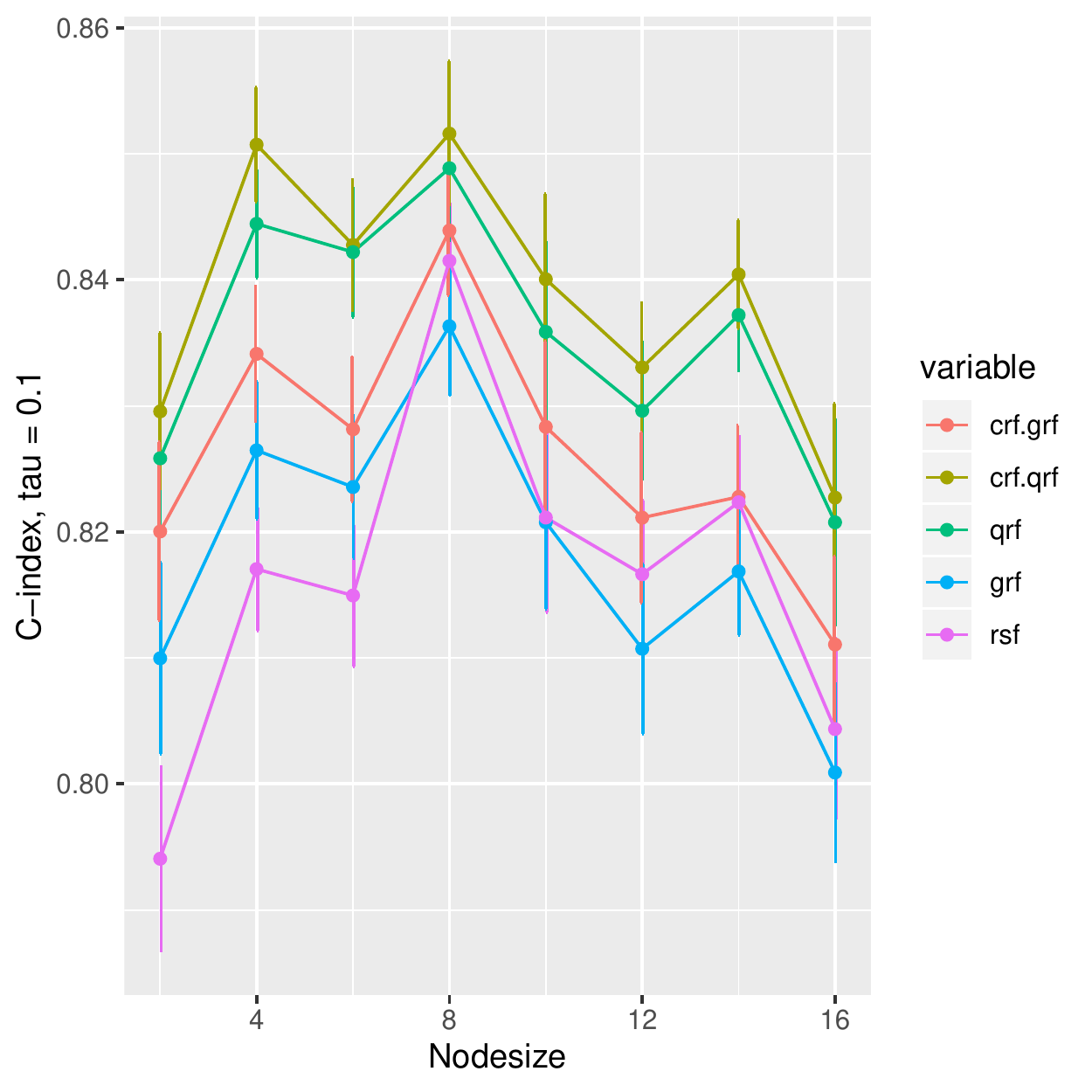}
        \caption{pbc: $\tau = 0.1$}
    \end{subfigure}
    ~
    \begin{subfigure}[b]{0.3\linewidth}
        \includegraphics[width=\textwidth]{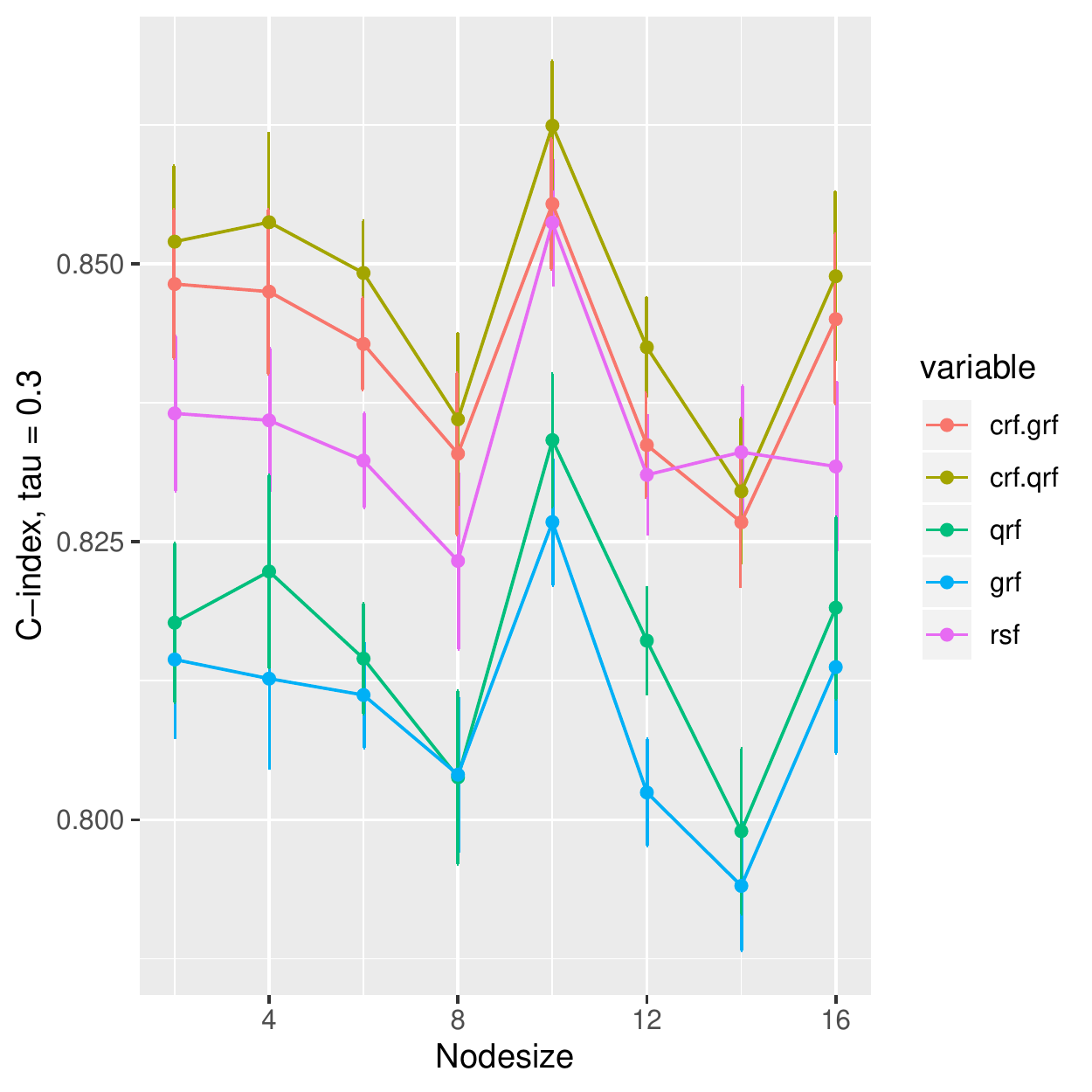}
        \caption{pbc: $\tau = 0.3$}
    \end{subfigure}
    ~
    \begin{subfigure}[b]{0.3\linewidth}
        \includegraphics[width=\textwidth]{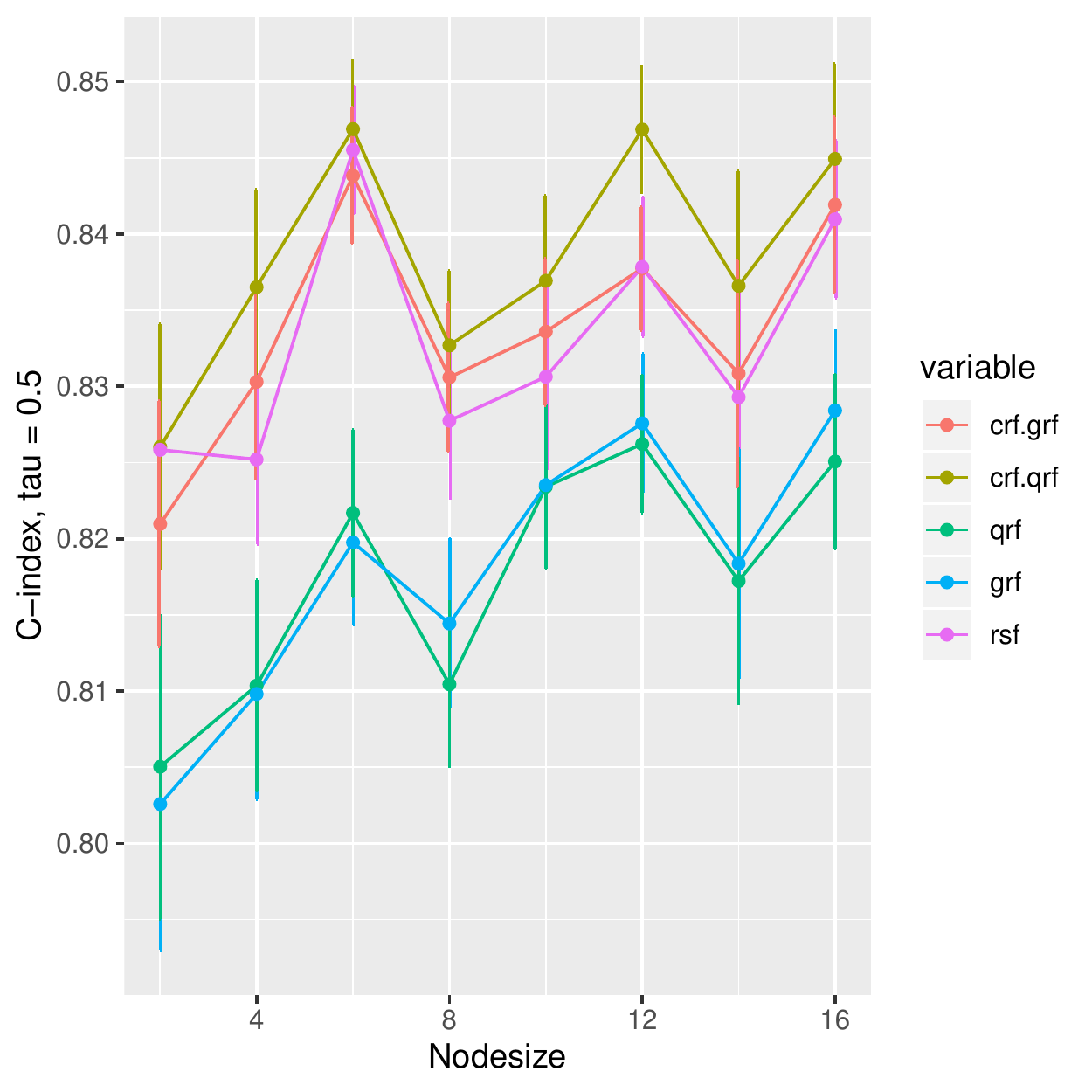}
        \caption{pbc: $\tau = 0.5$}
    \end{subfigure}
    ~
    \caption{The results on real datasets. On BostonHousing, we report the quantile losses (the lower the better), and on pbc dataset, we report the C-index (the higher the better).}
    \label{fig:real_data_nodesize}
    \vspace{-0.1in}
\end{figure*}

We observe that when $n$ increases, two curves become closer and are both good approximations of the true survival curve. But the first method \eqref{eq:KM_kNN} has an extra tuning parameter $k$ -- the number of nearest neighbors. Therefore, in the experiments, we always choose to use the second estimator \eqref{eq:Beran_rf} which is parameter free.

Note that the estimated survival function will degenerate at the tail of the distribution when the test point $x$ is small. This is a common phenomenon even for the regular KM estimator because there is no censored observations beyond some time point. In the AFT model, when $x$ is small, the conditional mean of $T$ is also small, and hence we could not observe most of the censoring values, leading to degenerated survival curves.

\subsection{Real Data}

In this section, we compare the proposed method with other forest algorithms on two real datasets, BostonHousing \citep{Dua:2019} and Primary Biliary Cirrhosis (PBC) Data \citep{fleming2011counting}. On the BostonHousing data, we manually generate censoring variables from $\text{Exp}(\lambda=1/2 \bar{y})$ where $\bar{y}$ is the sample mean of the house prices. The censoring level is about 40\%. We evaluate the models using quantile loss because we know the true responses in this case. The PBC data is already right-censored, and the censoring rate is about 60\%. On this dataset, we cannot use quantile loss for evaluation because we do not know the true responses of the censored data in the test set. (Note that we can sample uncensored data to form a test set, but these data will be biased.) Instead, we use the Harrell's concordance index (C-index) \citep{harrell1982evaluating}. We repeat the experiment for each node size for 50 times and report the mean and standard deviation of the C-index. For each experiment, we randomly sample 80\% of the data for training and the rest for testing. All the forests contain 1000 trees. The results are in Figure \ref{fig:real_data_nodesize}. 

Overall, the proposed method with random forest weights \textit{crf-quantile} has the best performance. It agrees with our observation in the simulations that \textit{crf-quantile} works better than the other methods if there is no clear heteroscedasticity in the data.

\section{Discussion}
In this article, we introduced \textit{censored quantile regression forest}, a novel non-parametric method for quantile regression problems that is integrated with the censored nature of the observations. While preserving information carried by the censored observations, the novel estimating equation maintains the flexibility of general forest approaches. One of the promising applications of the introduced method is in the estimation of heterogeneous treatment effects when the response variable is censored. Treatment discovery with right-censored observations is an important and yet poorly understood research area. Equipping this literature with the proposed fully non-parametric approach would lead to a significant broadening of the now more known parametric approaches. We also observe that our estimating equations can be easily replaced with another kind that targets treatment effects directly.

\clearpage
\bibliography{forestcqr}

\clearpage
\appendix
\section{Theorems and Proofs}
\subsection{Time complexity} \label{subsec:alg}
The step \ref{step:4} in Algorithm \ref{alg:main} involves of finding the $q^*$ in a candidate set $\mathcal{C}$ that sets the estimating equation $S_n(q; \tau)$ closest to zero. We simply evaluate the function $S_n(q;\tau)$ for all possible $q$ in $\mathcal{C}$ and find the minimum point. Note that for any fixed $\tau$, $S_n(q; \tau)$ is a step function in $q$ with jumps at $Y_i$'s because the discontinuities only happen at $Y_i$'s for $\hat{G}(q|x)$ (both \eqref{eq:KM_kNN} and \eqref{eq:Beran_rf}) and $\sum_{i=1}^n w(X_i, x) \ind(Y_i > q)$. Therefore, the candidate set $\mathcal{C} \subset \{Y_i\}_{i=1}^n$, and $|\mathcal{C}| = n$ in the worst case.

But in fact, for any fixed $x$, only $Y_i$'s with the corresponding feature vector $X_i \in R_x$ \eqref{eq:KM_kNN} or with $w(X_i, x) > 0$ \eqref{eq:Beran_rf} will be jump points, and hence, we can refine $\mathcal{C} = \{Y_i: X_i \in R_x\}$ for \eqref{eq:KM_kNN} or $\mathcal{C} = \{Y_i: w(X_i, x) > 0\}$ for \eqref{eq:Beran_rf}. We then have the following theorem.

\begin{theorem}\label{thm:complexity}
For a fixed test point $x$, depending on whether $G(q|X)$ is estimated by \eqref{eq:KM_kNN} or \eqref{eq:Beran_rf}, the time complexity for Algorithm \ref{alg:main} is $O(n \max\{k, \log(n)\})$ or $O(n m \log(n)^{p-1})$, respectively.
\end{theorem}

\begin{proof}[Proof of Theorem \ref{thm:complexity}]
To get the candidate set $\mathcal{C}$, if we use the k-nearest neighbor estimator \eqref{eq:KM_kNN}, then the first step is to sort $n$ weights and choose the largest $k$ elements. This is in general a $O(n \log(n))$ procedure. If we use the Beran estimator \eqref{eq:Beran_rf}, then the time complexity is $O(n)$ because we need to find all the nonzero weights.

After we have the candidate set $\mathcal{C}$, evaluating $S_n(q;\tau)$ for all $q \in \mathcal{C}$ and finding the minimum is a $O(n |\mathcal{C}|)$ procedure. For \eqref{eq:KM_kNN}, $|\mathcal{C}| = k$; and for \eqref{eq:Beran_rf}, $|\mathcal{C}|$ is in the order of $m \log(n)^{p-1}$ by \cite{lin2006random}.
\end{proof}

\subsection{Proof of Theorem \ref{thm:part1}}
\begin{proof}
When the conditions \ref{cond:3} to \ref{cond:5} are satisfied, by Theorem 3 in \citet{athey2019generalized} or Theorem 1 in \citet{meinshausen2006quantile}, we have
\begin{align*}
    \left| \sum_{i=1}^n w(X_i, x) \ind\{Y_i \le q\} - \mathbb{P}(Y \le q | x) \right| = o_p(1).
\end{align*}
Note that $\sum_{i=1}^n w(X_i, x) = 1$ and $0 \le w(X_i, x) \le 1/m$. For convenience, we suppress the dependency on $x$ and denote $F_n(q) = \sum_{i=1}^n w(X_i, x) \ind\{Y_i \le q\}$ and $F(q) = \mathbb{P}(Y \le q | x)$. Because $F$ is continuous, choose $q_0 < q_1 < \ldots < q_n$ from $\mathcal{B}$ such that $F(q_j) - F(q_{j-1}) = 1/n$. Then for any $q \in \mathcal{B}$, there exists $j \in \{1,\ldots, n\}$ such that $q \in [q_{j-1}, q_j]$, and hence $F_n(q) - F(q) \le F_n(q_j) - F(q_{j-1}) = F_n(q_j) - F(q_j) + 1/n$. Similarly, $F_n(q) - F(q) \ge F_n(q_{j-1}) - F(q_{j-1}) - 1/n$. Therefore, we have
\begin{align*}
    \sup_{q \in \mathcal{B}} \left| F_n(q) - F(q) | x) \right| \le \max_{j=1,\ldots,n} \left| F_n(q_j) - F(q_j) \right| + 1/n.
\end{align*}
Then by Bonferroni's inequality, we have
\begin{align*}
    \sup_{q \in \mathcal{B}} \left| F_n(q) - F(q) | x) \right| = o_p(1).
\end{align*}
Combined with Condition \ref{cond:5}, we have the expected result.
\end{proof}

\subsection{Proof of Theorem \ref{thm:consistency}}
\begin{proof}[Proof of Theorem \ref{thm:consistency}]
By \cite{van2000asymptotic}, we only need to show for any $\tau \in (0,1)$, $x \in \mathcal{X}$,
\begin{enumerate}
\item \label{part:1}
$\sup_{q \in [-r,r]} | S_n(q; \tau) - S(q; \tau) | = o_p(1)$.
\item \label{part:2}
For any $\epsilon > 0$,
$\inf\{|S(q;\tau)|: |q - q^*| \ge \epsilon, q \in [-r, r] \} > 0$. Here, $q^*$ stands for the true $\tau$th quantile of $T$.
\item \label{part:3}
$S_n(q_n; \tau) = o_p(1)$.
\end{enumerate}
Part \ref{part:1} has been proved by Theorem \ref{thm:part1}. For part \ref{part:2}, note that
\begin{eqnarray*}
S(q;\tau) &=& (1-\tau)G(q|x) - \mathbb{P}(Y > q|x) \\
&=& (1-\tau)G(q|x) - \mathbb{P}(T > q|x) \mathbb{P}(C > q|x) \\
&=& ((1-\tau) - \mathbb{P}(T > q|x)) G(q | x) \\
&=& (\mathbb{P}(T \le q|x) - \tau) G(q | x).
\end{eqnarray*}
The second equality is because of the conditionally independency between $T$ and $C$. Fix an $\epsilon > 0$, and denote $$E = \{|S(q;\tau)|: |q - q^*| \ge \epsilon, q \in [-r, r]\}.$$ Since $0 < \tau < 1$, by Condition \ref{cond:4}, there exists some $l > 0$ such that $G(q|x) \ge l$ and $$|\mathbb{P}(T \le q | x) - \tau| \ge l$$ for $q \in E$. Now for part \ref{part:3}, by the definition of $q_n$, we know 
$$|S_n(q_n; \tau)| = \min_{q \in [-r,r]} |S_n(q;\tau)|.$$ Also by definition of $q^*$, 
$$0 = |S(q^*;\tau)| = \min_{q \in [-r,r]} |S(q;\tau)|.$$ Then we get
\begin{align*}
    & |S_n(q_n;\tau)| \\
    & = |S_n(q_n;\tau)| - |S_n(q^*;\tau)| + |S_n(q^*;\tau)| - |S(q^*;\tau)| \\
    &\le | S_n(q^*;\tau) - S(q^*;\tau) | \\
    &\le \sup_{q \in [-r,r]} |S_n(q;\tau) - S(q;\tau)| \\
    &= o_p(1)
\end{align*}
where the first inequality is because of the definition of $q_n$ and the triangular inequality.
\end{proof}

\section{More Experiments}

\subsection{Prediction Intervals}
All the forest methods can be used to get $95\%$ prediction intervals by predicting the $0.025$ and $0.975$ quantiles of the true response variable. Then for any location $x \in \mathcal{X}$, a straightforward confidence interval will be $[Q(x;0.025), Q(x;0.975)]$. The result is illustrated in Figure \ref{fig:sine_ci_plot} for the case of univariate censored sine model. For each data set, we bootstrap the data and calculate the $0.025$ and $0.975$ quantile for the out of bag points. Then for each node size, we repeat this process for 20 times and calculate the average coverage rate of the confidence intervals.

\begin{figure}[!htb]
    \small
    \centering
    \begin{subfigure}[b]{0.3\linewidth}
        \centering
        \includegraphics[width=\textwidth]{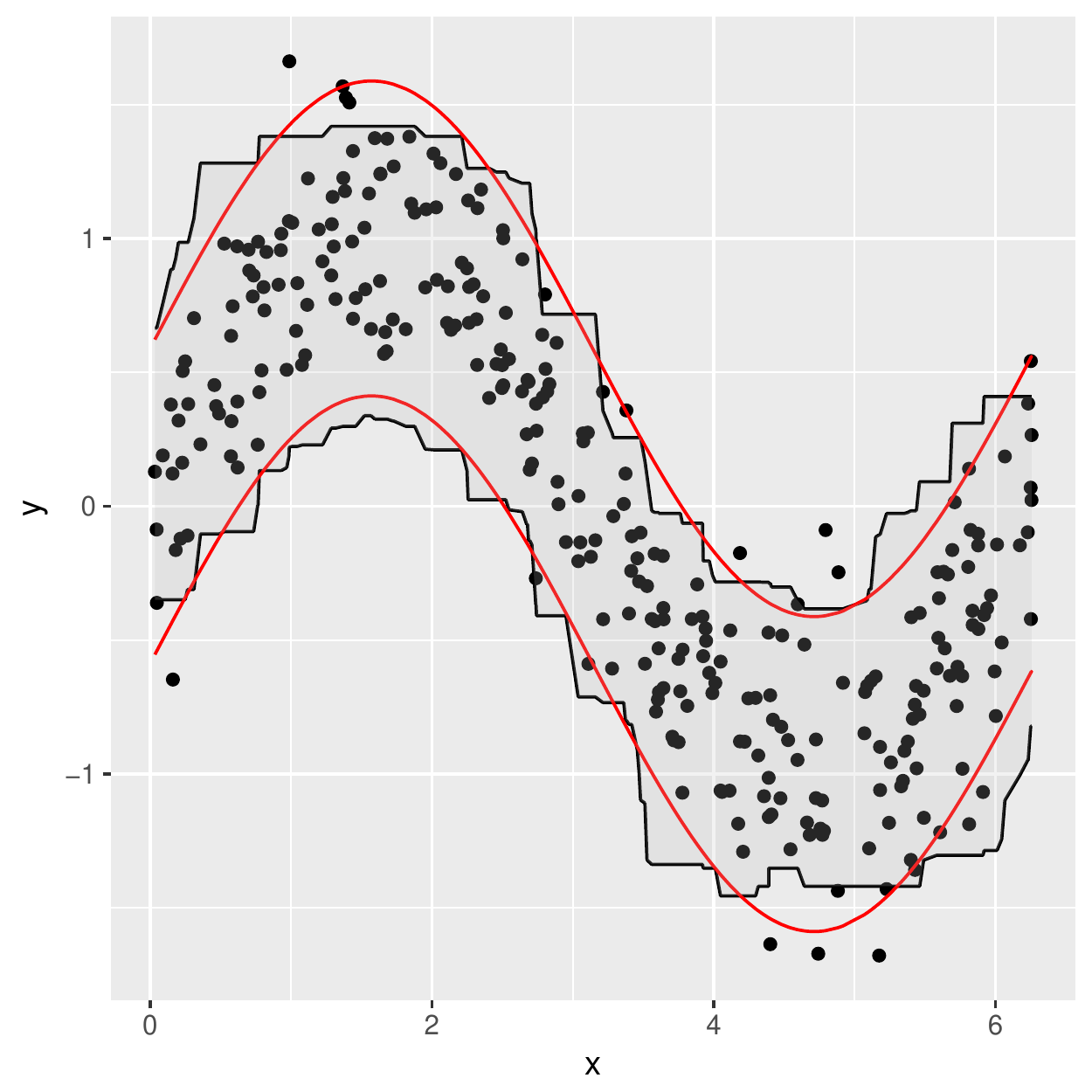}
        \caption{crf}
    \end{subfigure}
    ~
    \begin{subfigure}[b]{0.3\linewidth}
        \centering
        \includegraphics[width=\textwidth]{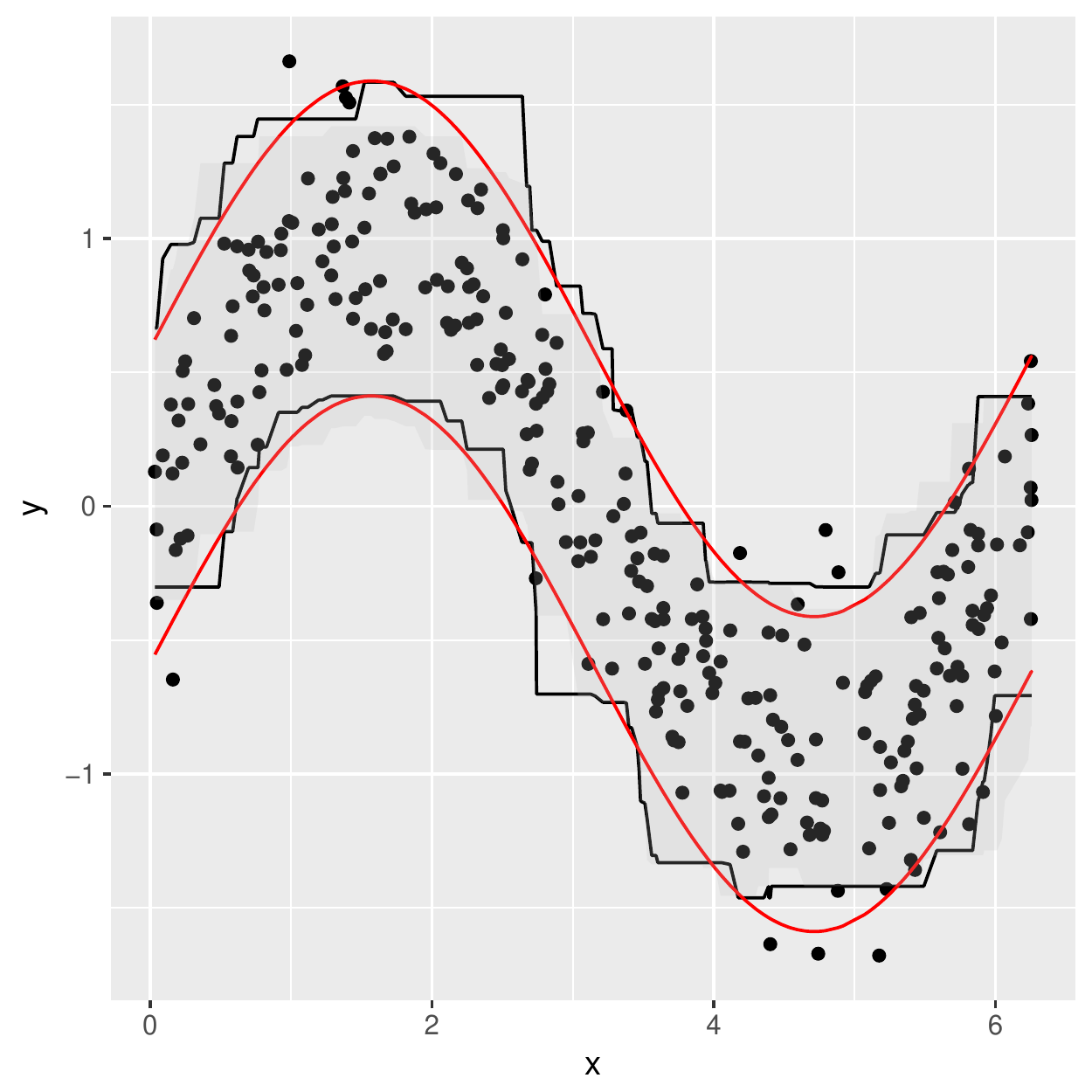}
        \caption{qrf-oracle}
    \end{subfigure}
    ~
    \begin{subfigure}[b]{0.3\linewidth}
        \centering
        \includegraphics[width=\textwidth]{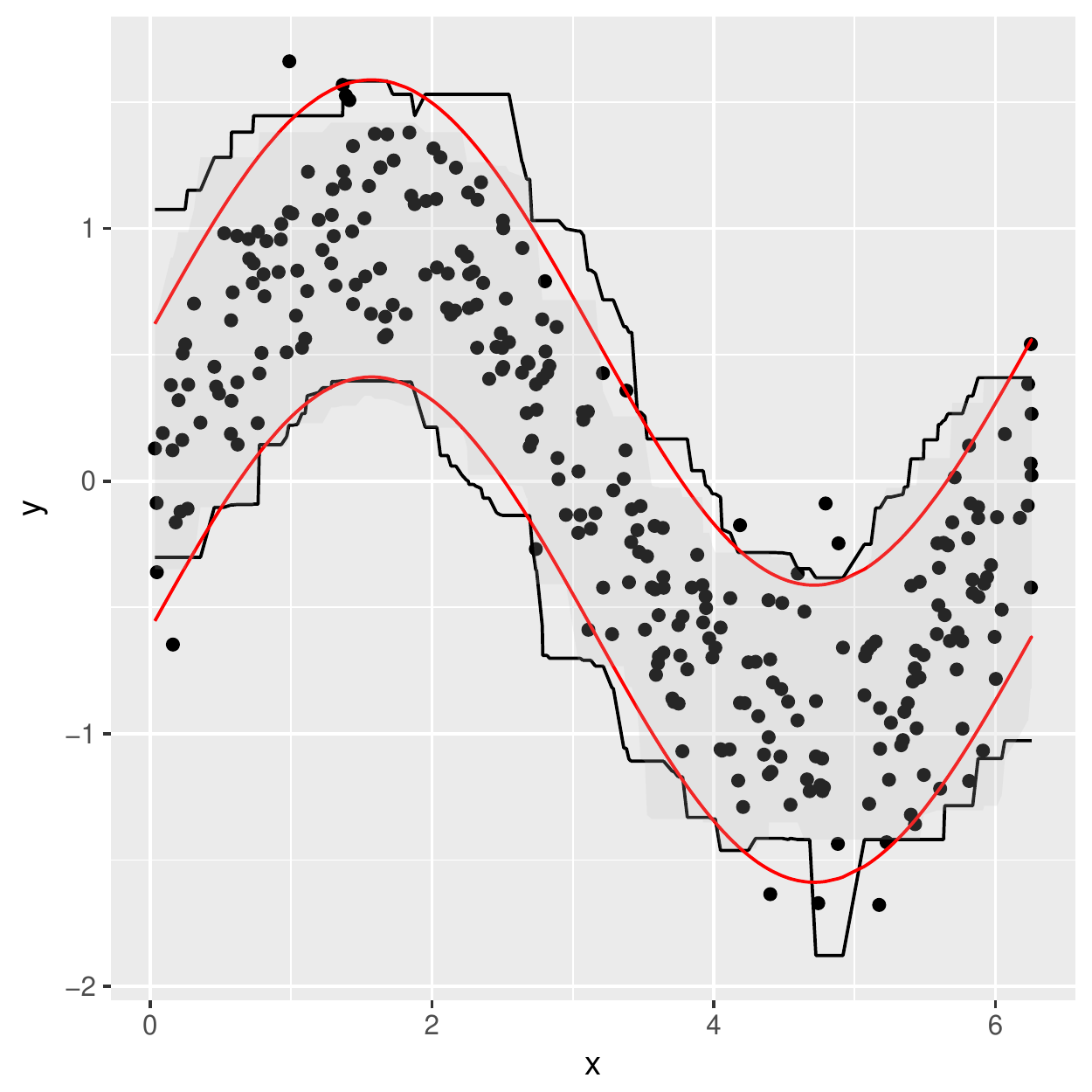}
        \caption{grf-oracle}
    \end{subfigure}
    ~
    
    \vspace{-0.05in}
    
    \begin{subfigure}[b]{0.3\linewidth}
        \centering
        \includegraphics[width=\textwidth]{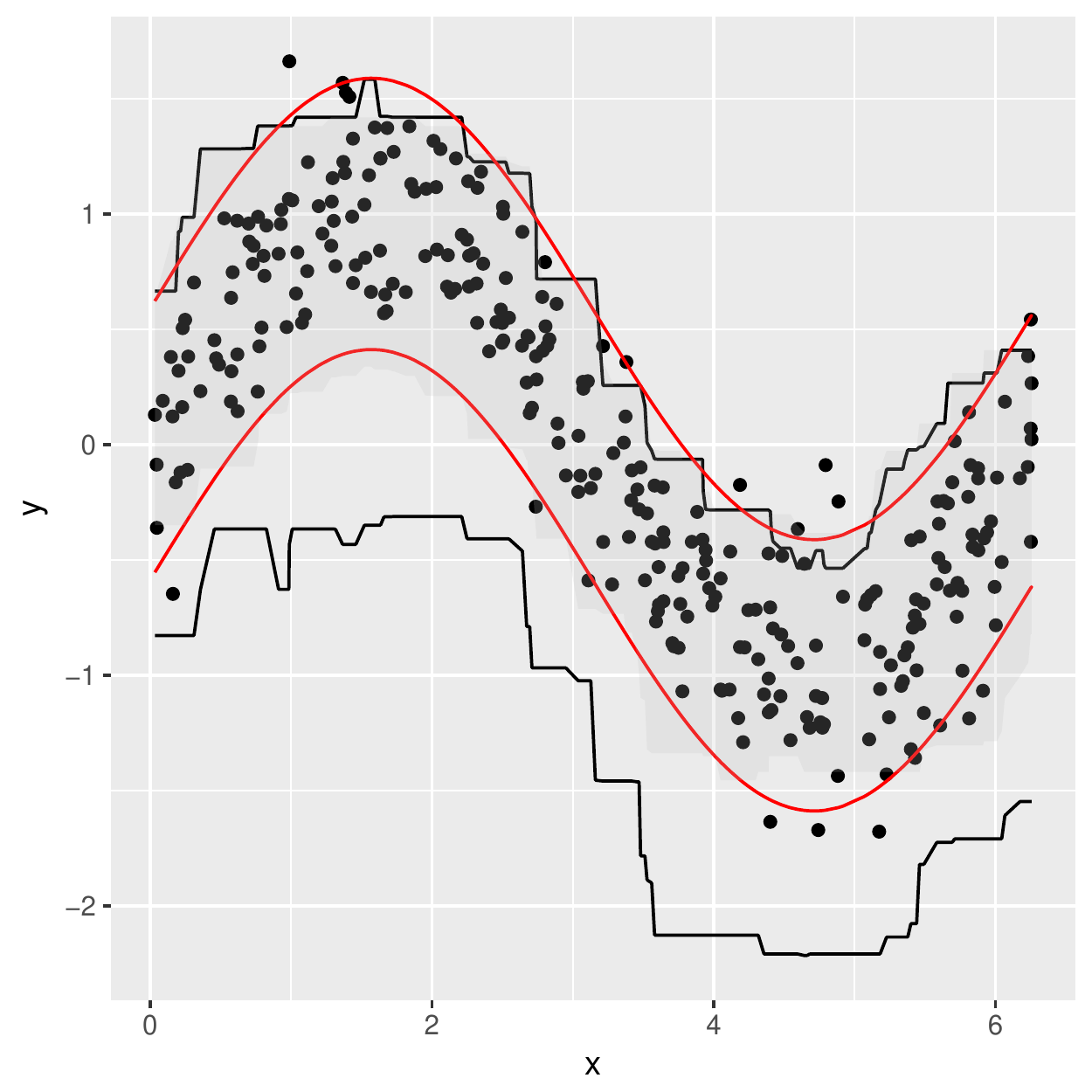}
        \caption{qrf}
    \end{subfigure}
    ~
    \begin{subfigure}[b]{0.3\linewidth}
        \centering
        \includegraphics[width=\textwidth]{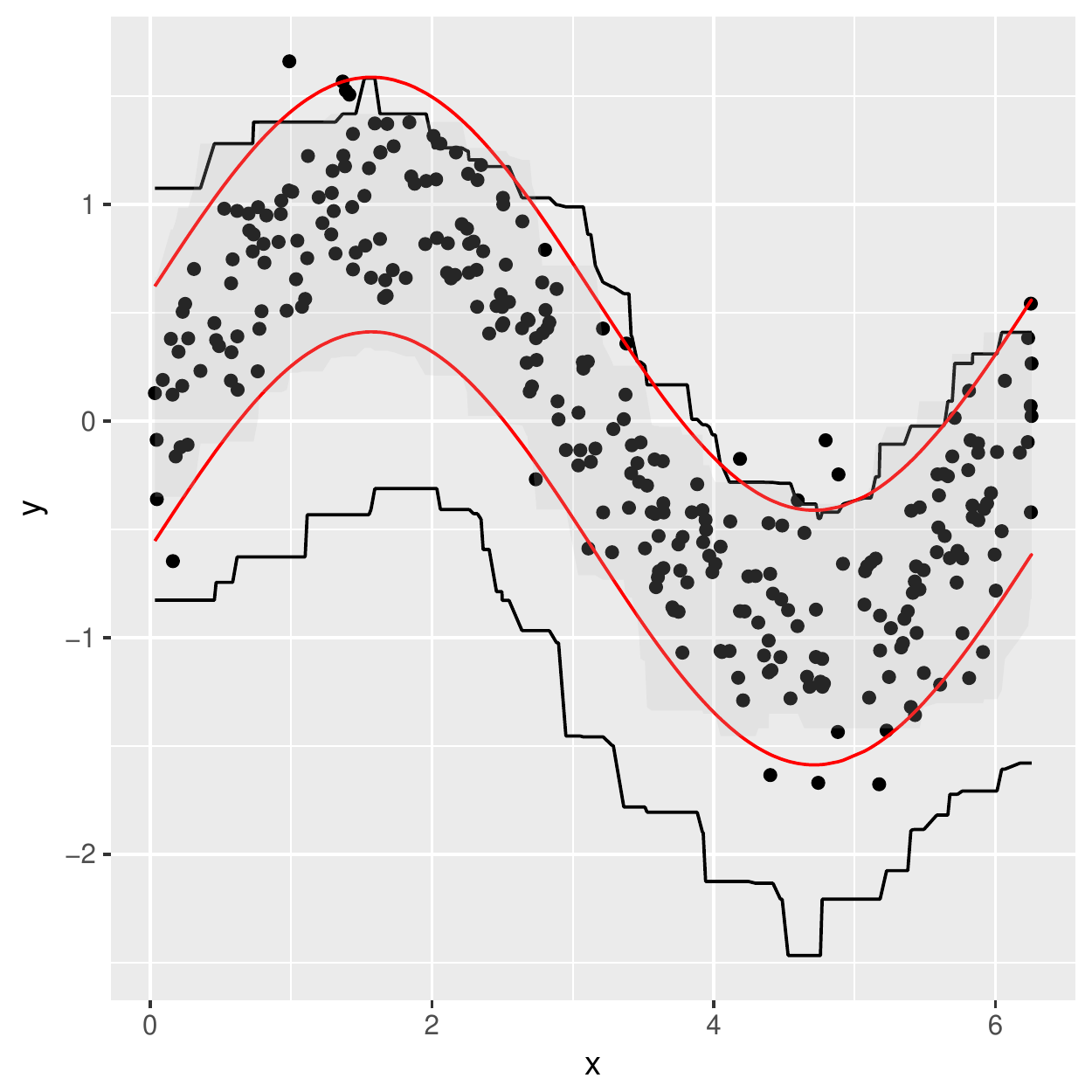}
        \caption{grf}
    \end{subfigure}
    
    \caption{Prediction intervals of the univariate censored since model. We observe that in all of the cases, our method \textit{crf} and \textit{qrf-oracle} give the coverage closest to $95\%$. Both \textit{qrf} and \textit{grf} perform much worse on predicting lower quantiles. They tend to under-estimate the lower quantiles and hence make the confidence intervals much wider than the true ones.}
    \label{fig:sine_ci_plot}
\end{figure}

    

\subsection{One-dimensional Sine-curve Model}

\begin{figure}[!htb]
    \centering
    \includegraphics[width=\columnwidth]{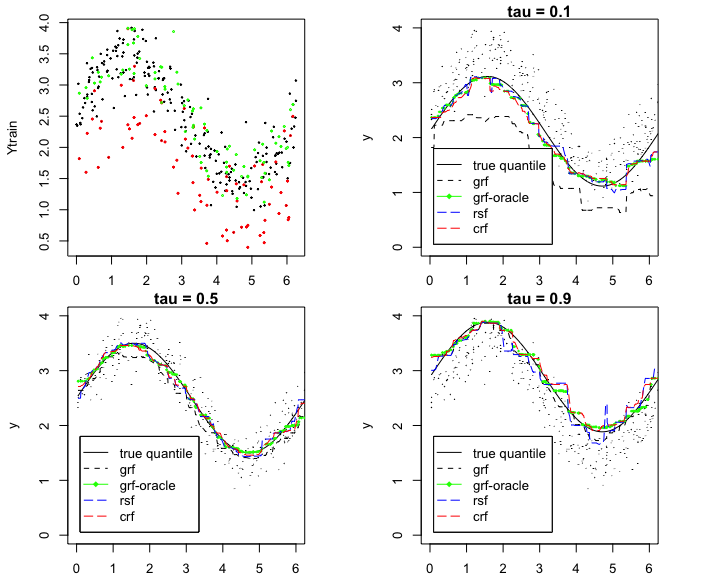}
    \caption{One-dimensional Sine model results.}
    \label{fig:one_dim_sine}
\end{figure}

Since the proposed method \textit{crf} is nonparametric and does not rely on any parametric assumption, it can be used to estimate quantiles for any general model $T = f(X) + \epsilon$. Hence we set $f(x) = \sin(x)$ and
\begin{equation*}
    T = 2.5 + \sin(X) + \epsilon
\end{equation*}
where $X \sim \textrm{Unif}(0,2 \pi)$ and $\epsilon \sim \mathcal{N}(0, 0.3^2)$. The censoring variable $C \sim 1 + \sin(X) + \textrm{Exp}(\lambda = 0.2)$ depends on the covariates, and the censoring level is about $25\%$. The results are in Figure \ref{fig:one_dim_sine}. 

Again, the proposed model \textit{crf} produces almost identical quantile predictions compared with \textit{grf-oracle}. Especially when $\tau=0.1$, the $grf$ result (blue dotted curve) severely deviates from the true quantile, while \textit{crf} still predicts the correct quantile and performs as good as the oracle \textit{grf-oracle}.

\clearpage

\begin{figure*}[!htb]
    \centering
    \vspace{-0.05in}
    \begin{subfigure}[b]{0.30\linewidth}
        \includegraphics[width=\textwidth]{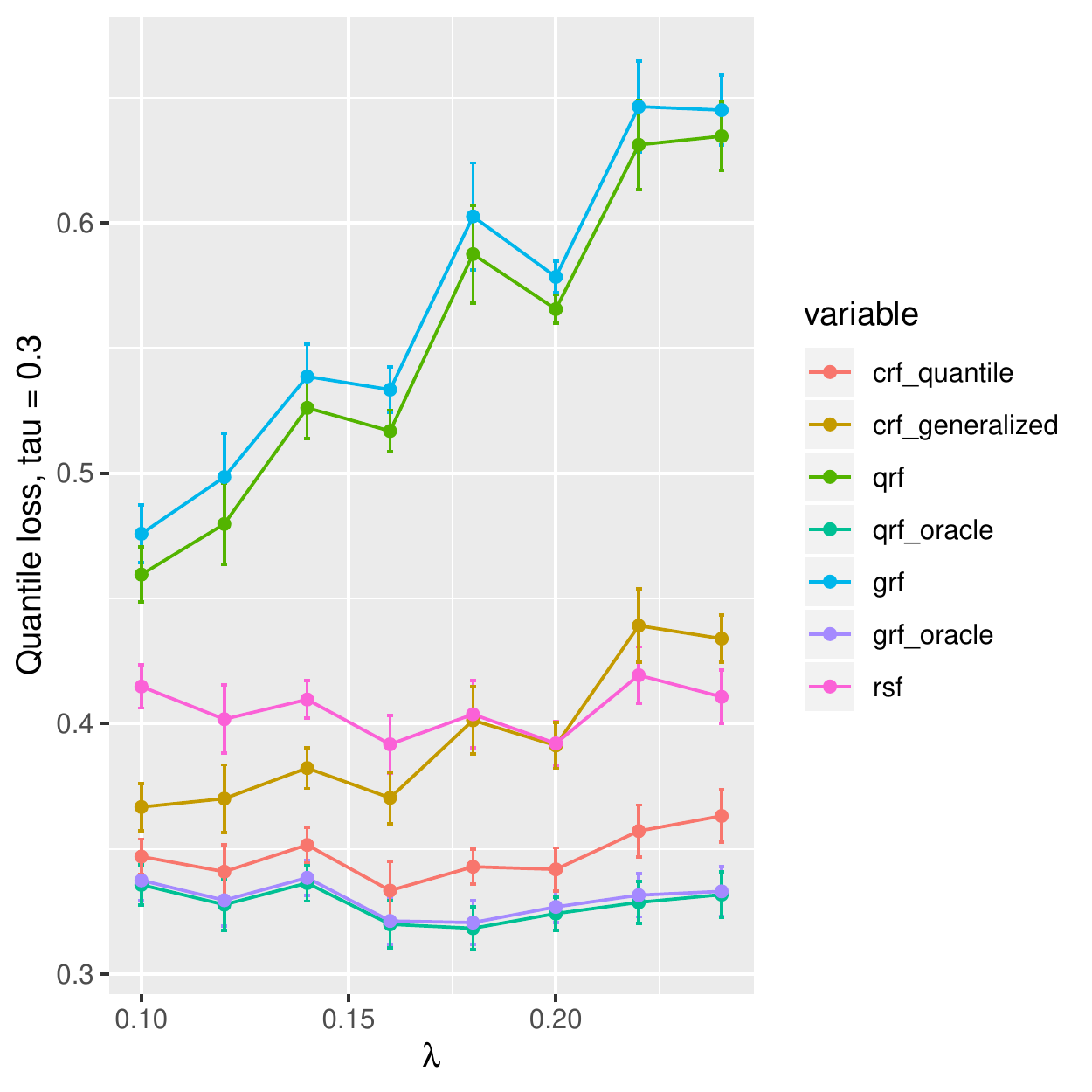}
        
    \end{subfigure}
    ~
    \begin{subfigure}[b]{0.30\linewidth}
        \includegraphics[width=\textwidth]{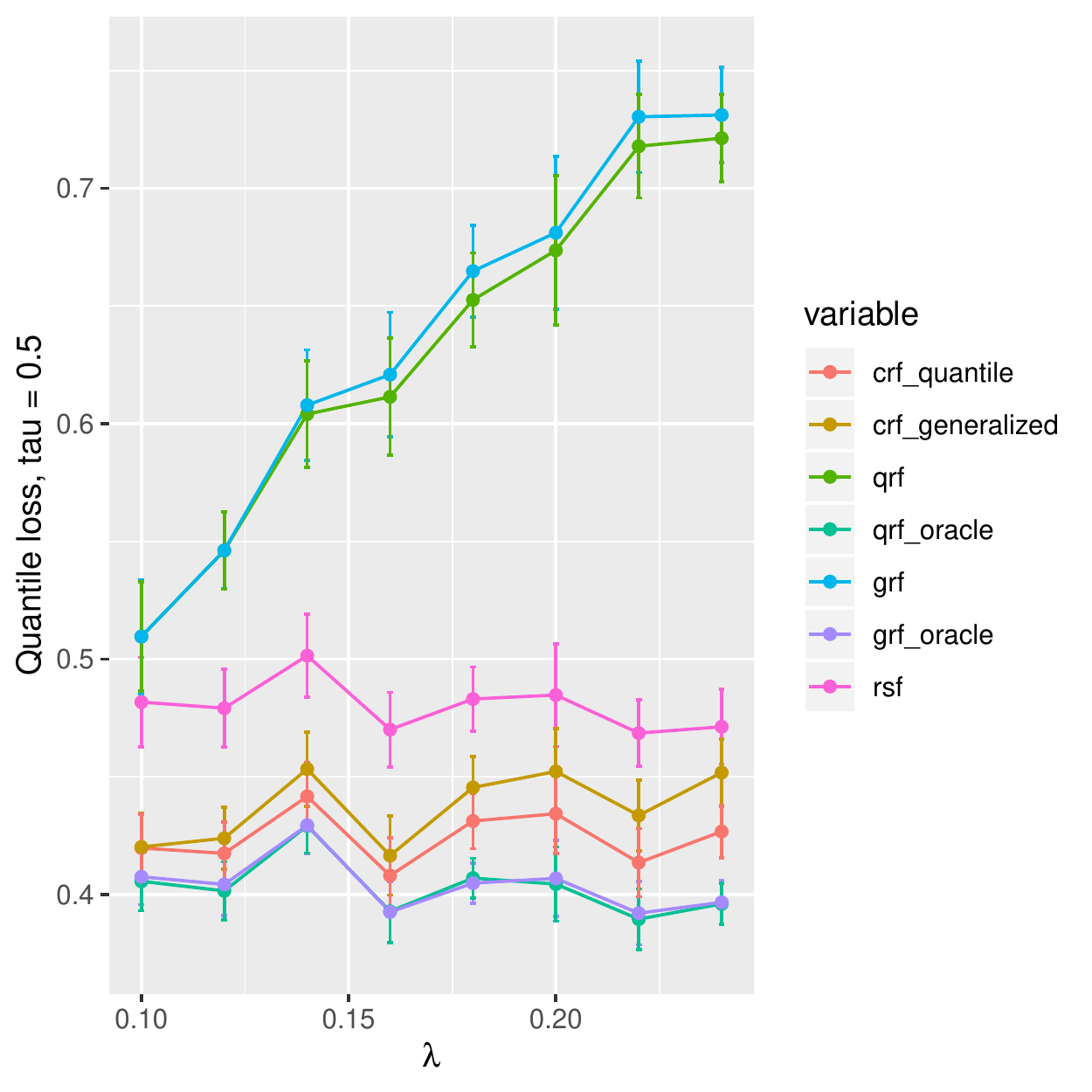}
        
    \end{subfigure}
    ~
    \begin{subfigure}[b]{0.30\linewidth}
        \includegraphics[width=\textwidth]{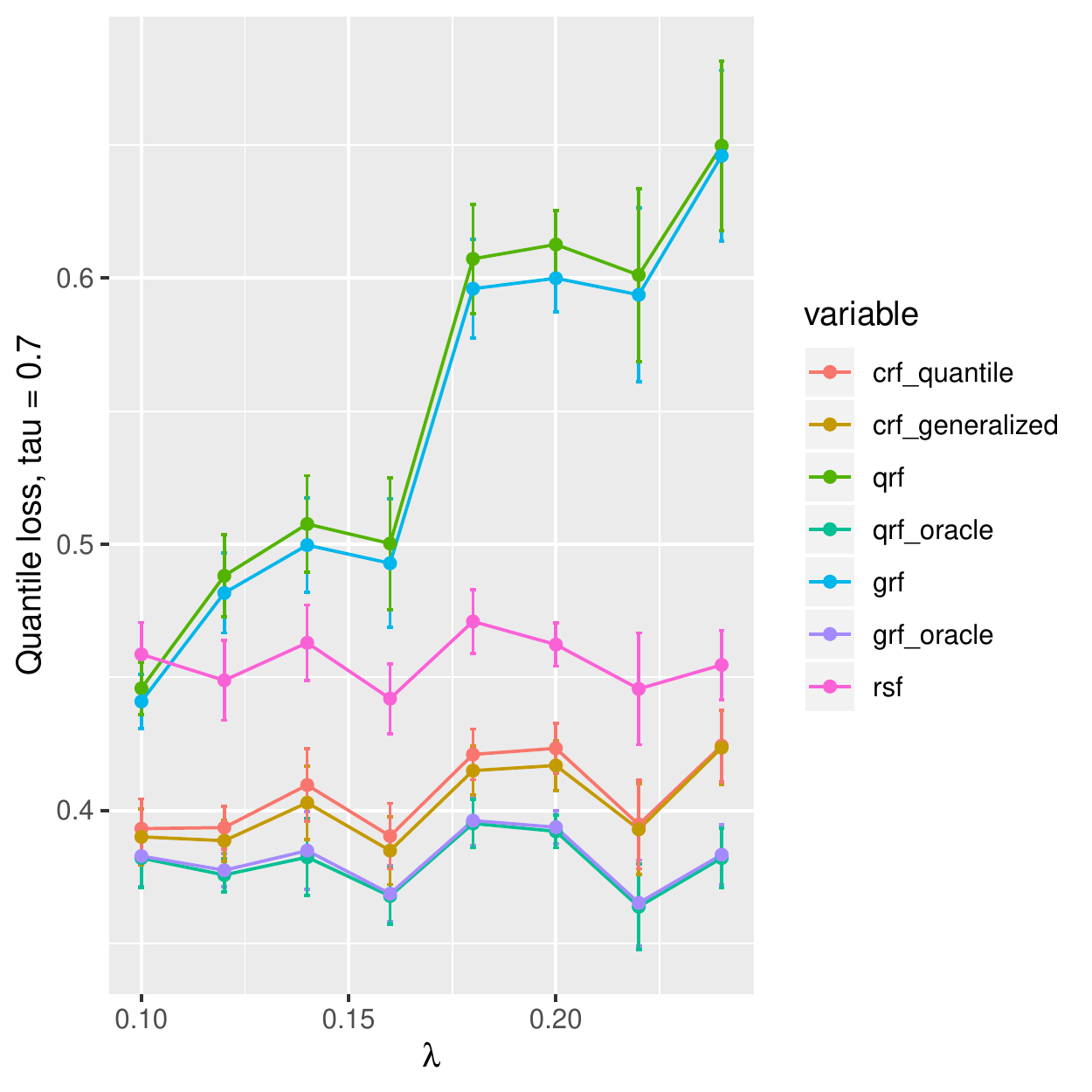}
        
    \end{subfigure}
    
    \begin{subfigure}[b]{0.30\linewidth}
        \includegraphics[width=\textwidth]{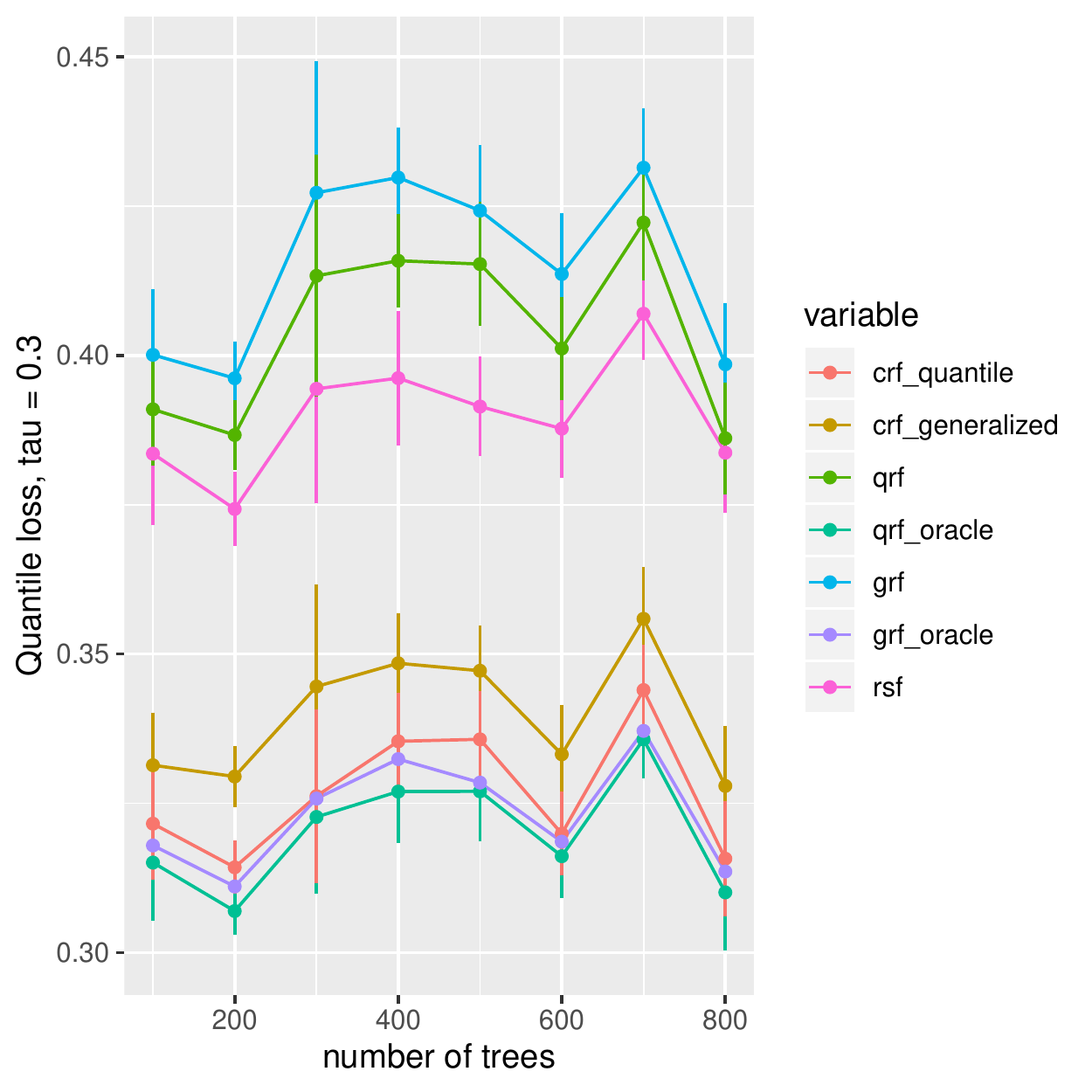}
        
    \end{subfigure}
    ~
    \begin{subfigure}[b]{0.30\linewidth}
        \includegraphics[width=\textwidth]{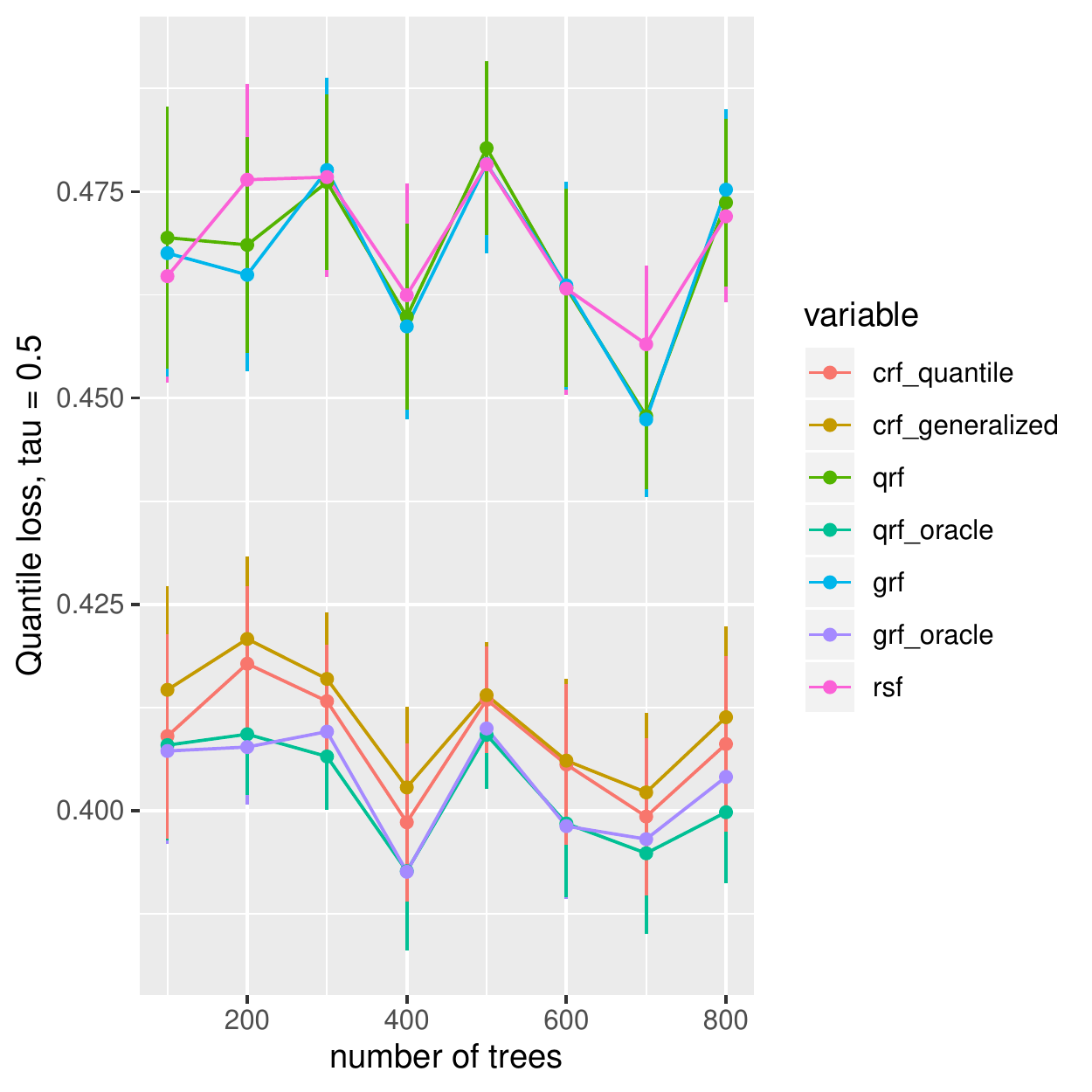}
        
    \end{subfigure}
    ~
    \begin{subfigure}[b]{0.30\linewidth}
        \includegraphics[width=\textwidth]{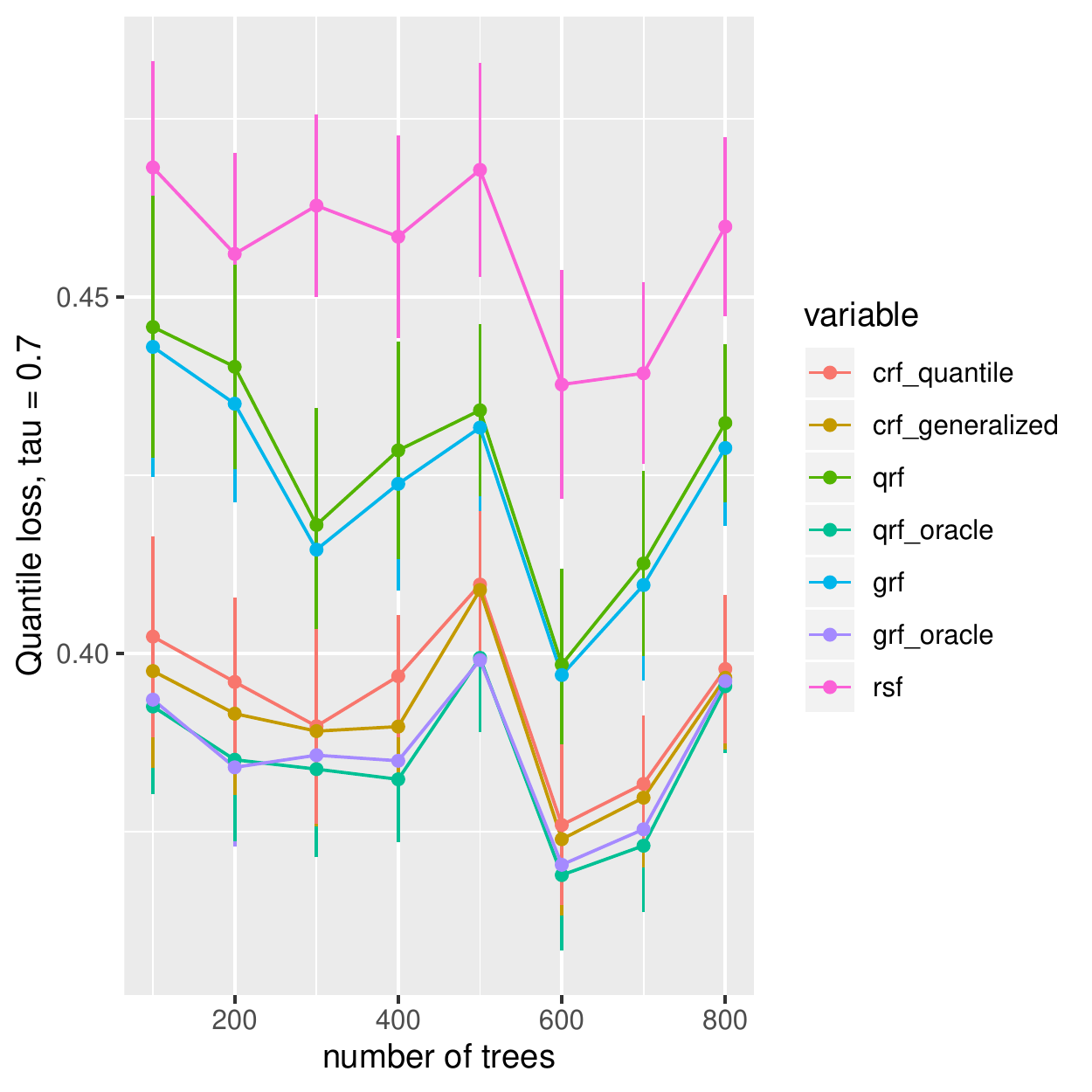}
        
    \end{subfigure}
    
    \vspace{-0.05in}
\end{figure*}
\end{document}